\documentclass{article}
\usepackage{times}
\usepackage{subfiles}
\usepackage{fontawesome}
\usepackage[utf8]{inputenc}
\usepackage[english]{babel}
\usepackage{import}
\usepackage[ruled,vlined,linesnumbered]{algorithm2e}
\SetKwInOut{Input}{Input}
\SetKwInOut{Output}{Output}
\usepackage{graphicx}
\usepackage{csquotes}

\usepackage[dvipsnames]{xcolor}

\definecolor{mygreen}{rgb}{0,0.4,0}

\newcommand{\both}{$p\underline{q}$}
\newcommand{\onlyP}{$p \underline{\neg q}$}
\newcommand{\onlyQ}{$\neg p\underline{q}$}
\newcommand{\neither}{$\neg p\underline{\neg q}$}

\usepackage{bm}
\usepackage{float}
\usepackage[margin=3.2cm]{geometry}
\usepackage{subcaption}
\usepackage{amsmath, amsthm,mathtools,amssymb,comment}
\usepackage{stackengine} 

\usepackage{tikz}
    \tikzstyle{dnode}=[inner sep=1pt,outer sep=1pt,draw,circle,minimum width=9pt]
    \tikzstyle{nnode}=[inner sep=1pt,outer sep=1pt,circle,minimum width=9pt]
    \tikzstyle{label-edge}=[midway,fill=white, inner sep=1pt]
    \tikzstyle{w}=[draw,fill,circle,inner sep=1.5pt]

\tikzstyle{dnode}=[inner sep=1pt,outer sep=1pt,draw,circle,minimum width=9pt]
\tikzstyle{nnode}=[inner sep=1pt,outer sep=1pt,circle,minimum width=9pt]

\usetikzlibrary{shapes}
\usetikzlibrary{positioning, fit, shapes.geometric, matrix}

\usepackage{forest}
\forestset{
sn edges/.style={for tree={circle, draw, thick, minimum size=1cm,inner sep=1,edge={->,thick}}}
}

\newcommand{\ldel}{\mathcal{L}_{\text{DEL}}}

\usepackage{graphicx}
    \graphicspath{{./figures/}{./sections/figures/}}

\usepackage[backend=biber]{biblatex}
\bibliography{bib.bib}

\newcommand{\bisim}{{\underline{\leftrightarrow}}}

\newcommand{\post}{\mathsf{post}}

\newcommand{\states}{S}
\newcommand{\obs}{\Omega}
\newcommand{\transmod}{T}
\newcommand{\obsmod}{Obs}

\newcommand{\actions}{A}

\newtheorem{theorem}{Theorem}

\newtheorem{proposition}{Proposition}
\newtheorem{lemma}{Lemma}
\theoremstyle{definition}
\newtheorem{definition}{Definition}
\newtheorem{example}{Example}
\theoremstyle{remark}

\newcommand{\tobo}[1]{{\bfseries\color{cyan}[Thomas: #1]}}

\title{Learning to Act and Observe \\ in Partially Observable Domains}
\author{Thomas Bolander, Nina Gierasimczuk, Andr\'es Occhipinti Liberman}
\date{}

\begin{document}

\maketitle

\begin{abstract}
We consider a learning agent in a partially observable environment, with which the agent has never interacted before, and about which it learns both what it can observe and how its actions affect the environment. The agent can learn about this domain from experience gathered by taking actions in the domain and observing their results. We present learning algorithms capable of learning as much as possible (in a well-defined sense) both about what is directly observable and about what actions do in the domain, given the learner's observational constraints. We differentiate the level of domain knowledge attained by each algorithm, and characterize the type of observations required to reach it. The algorithms use dynamic epistemic logic (DEL) to represent the learned domain information symbolically. Our work continues that of Bolander and Gierasimczuk (2015), which developed DEL-based learning algorithms based to learn domain information in fully observable domains.
\end{abstract}

\section{Introduction}

This paper explores a learning task which we refer to as \emph{domain learning}. Domain learning involves learning a representation of the dynamics of a domain (sometimes called an environment), from the experience gathered by performing actions in this domain and observing their results. We are concerned with domain learning in what we call unknown, partially observable domains. We introduce the learning task and the main results informally in this section, leaving formal details for later. 

Consider an agent inhabiting an unknown, partially observable domain. By an \emph{unknown} domain, we mean one that the agent hasn't interacted with before, so it doesn't know how its actions may affect the domain. For example, a river would be an unknown domain to an agent that has never seen or been to a swimming pool, the sea, a river, or any other body of water. The agent would indeed not know how its actions (moving the limbs, breathing, etc.) would interact with this domain. By a \emph{partially observable} domain, we mean one in which the agent may get to see only a part of the domain at any point in time. For example, the agent in the river may be able to see underwater only within two meters of distance, because the water is murky. Although the agent starts without any knowledge about how it can act in this domain, and what it can or cannot observe about it, we assume that it has access to experiences of interaction with the domain, gathered by trying out actions and observing what happens as a result. The goal of the agent is to \emph{learn a correct representation of what actions do, and what it can and cannot observe, from such experience}. We present learning algorithms that enable the agent to learn ``as much as possible'' about the domain, given the observational limits, provided that the learner is given a sufficient number of interactions to learn from. Of course, the exact meanings of ``learning as much as possible'', ``sufficient number of interactions'', etc., are important here, and will be made precise in the coming sections.

We are interested in learning representations of domains that rely on \emph{dynamic epistemic logic} (DEL). We therefore introduce DEL in Section \ref{sec:DEL}, where we also motivate our choice of this representation. Section \ref{sec:domains} presents domains and their properties formally. Sections \ref{sec:explicit}-\ref{sec:implicit} contain the main results of the paper. In Section \ref{sec:explicit}, we first introduce a distinction between what we call learning \emph{explicit} domain knowledge and learning \emph{implicit} domain knowledge. We define learning explicit domain knowledge as learning to know what will be directly observed when actions are executed. We contrast this with implicit knowledge, which includes what might additionally be inferred from the history of earlier actions and from the general experience with the domain. Section \ref{sec:explicit} focuses on learning explicit domain knowledge. We motivate this learning goal, formalise it, present a learner that achieves explicit domain knowledge, and we characterise the observations required to do so. Section \ref{sec:extending} extends learning beyond explicit knowledge, illustrating situations in which an agent can learn more about the domain than what can be directly observed in each state. Section \ref{sec:implicit} is focused on learning implicit domain knowledge. We formalise the goal of learning implicit knowledge, present a learner that attains this goal, and characterise the type of observations that are sufficient for doing so. Section \ref{sec:related} presents related work and Section \ref{sec:conclusion} concludes with final remarks and possible avenues for future research.

\section{Dynamic epistemic logic (DEL)\label{sec:DEL}}
Given a finite set $P$ of propositional symbols (atomic propositions), we define the (single-agent) \emph{dynamic epistemic language} $\ldel(P)$ over $P$ by the following BNF: 
\[
\phi ::=  p ~|~ \neg \phi ~|~ \phi \land \phi ~|~ K\phi \mid [\mathcal{E}] \phi,
\]
where $p \in P$ and $\mathcal{E}$ denotes an event model as defined below. We read $K\phi$ as ``it is known that $\phi$'' and $[\mathcal{E}] \phi$ as ``executing the action $\mathcal{E}$ necessarily leads to an epistemic model where $\phi$ holds''. By means of the standard abbreviations we introduce the additional symbols $\to$, $\vee$, $\leftrightarrow$, $\bot$, and $\top$. The modalities $[\mathcal{E}]$ are called \emph{dynamic modalities}. Formulas without dynamic modalities are called \emph{static formulas}. 
\begin{definition} 
An \emph{epistemic model} over a set of propositional symbols $P$ is $\mathcal{M} = (W,R,V)$, where 
\begin{itemize}
  \item $W$ is a finite set of \emph{worlds}
  \item $R\subseteq W \times W$ is an equivalence relation called the \emph{indistinguishability relation}
  \item $V: W \to 2^P$  is a \emph{valuation function} (or \emph{labelling function}, as it maps worlds to valuations)
  \end{itemize}
\end{definition}
Dynamic epistemic logic (DEL) introduces the concept of an event model for modelling the changes to states brought about by executions of actions \cite{baltag2004logics}. We here use a variant that includes (boolean) postconditions~\cite{ditmarsch2008semantic}, which means that actions can have both ontic effects (changing the factual states of affairs) and epistemic effects (changing the beliefs of agents).
\begin{definition} 
An \emph{event model} over $P$ is $\mathcal{E} = (E,Q,pre,post)$, where
\begin{itemize}
  \item 
$E$ is a finite set of \emph{events}
  \item 
$Q \subseteq E \times E$ is an equivalence relation called the \emph{indistinguishability relation}
  \item 
$pre: E \to \ldel(P)$ assigns to each event a \emph{precondition}
  \item 
$post: E \to (P \to \ldel(P))$ assigns to each event a \emph{postcondition mapping} mapping each event $e$ into a \emph{postcondition} $post(e)$. In this paper, postconditions are boolean, meaning that for each event $e$ and proposition $p$, $post(e)(p) \in \{ \top, \bot, p \}$ ($p$ is set true, false or unchanged).  
\end{itemize}
\end{definition}

Intuitively, events correspond to the ways in which an action changes the epistemic model, and the indistinguishability relation codes (an agent's) inability to recognize the difference between those different ways. 
In an event $e$, $pre(e)$ specifies what conditions have to be satisfied for it to take effect, and $post(e)$ specifies its outcome.

\begin{example}\label{exam:hospital1}
Consider the action of tossing a coin. It can be represented by the following event model over $P = \{h \}$, where $h$ means that the coin is facing heads up:
\[
 \mathcal{E} \ =   \quad 
\raisebox{-2mm}{
  \begin{tikzpicture}[>=latex,every loop/.style={->},minimum size=1.5mm,every node/.style={auto}]
    \node[label={below:$e_1\!:\langle \top, h \rangle$},w] (w1) at (0,0) {}; 
    \node[label={below:$e_2\!:\langle \top, \neg h \rangle$},w] (w2) at (3,0) {}; 
\end{tikzpicture} 
}
\]
We label each event $e$ by a pair whose first argument is the event's precondition while the second is its postcondition represented compactly as a sequence $ l_1\cdots~l_n$ of propositional literals given by: if $post(e)(p) = \top$ then $p$ is one of the $l_i$; if $post(e)(p) = \bot$, then $\neg p$ is one of the $l_i$; if $post(e)(p) = p$ then $p$ doesn't occur in $l_1\cdots~l_n$. Hence, formally we have $\mathcal{E} = (E,Q,pre,post)$ with $E = \{e_1,e_2\}$, $Q$ is the identity on $E$, $pre(e_1) = pre(e_2) = \top$, $post(e_1)(h) = \top$ and $post(e_2)(h) = \bot$. 
The event model encodes that tossing the coin will either make $h$ true ($e_1$) or $h$ false ($e_2$).

\end{example}

Given an epistemic model $\mathcal{M} = (W,R,V)$ and a world $w \in W$, truth of an epistemic formula $\phi$ in $w$ of $\mathcal{M}$ is defined as follows:
\[
\begin{array}{lp{5mm}cp{5mm}l}
  \mathcal{M},w \models p && \text{iff}  && p \in V(w) \\
  \mathcal{M},w \models \neg \phi &&\text{iff} &&\mathcal{M},w \not\models \phi \\
  \mathcal{M},w \models \phi \wedge \psi &&\text{iff} &&\mathcal{M},w \models \phi \text{ and } \mathcal{M},w \models \psi \\
  \mathcal{M},w \models K \phi &&\text{iff} &&\text{for all $v\in W$, if $w R v$ then  $\mathcal{M},v \models \phi$} \\
  \mathcal{M}, w \models [\mathcal{E}] \phi &&\text{iff} &&\text{for all events $e$ in $\mathcal{E}$,} \\
  && &&\text{\quad if $\mathcal{M}, w \models pre(e)$ then $\mathcal{M} \otimes \mathcal{E}, (w,e) \models \phi$} 
 \end{array}\]
When $\mathcal{M}, w \models \phi$ for all $w \in W$, we write $\mathcal{M} \models \phi$. Our syntax and semantics of the dynamic modality, $[\mathcal{E}]$, is a bit non-standard. Normally, one considers \emph{pointed} event models $(\mathcal{E},e)$, where $e$ is an event of $\mathcal{E}$ called the \emph{actual event}, and then provides a semantics for a dynamic modality of the form $[\mathcal{E},e]$. In our setting, event models are going to be used by agents to represent their uncertainty about the dynamics of an action they are trying to learn. For instance, an agent could use an event model with two events $e_1$ and $e_2$ to represent that it doesn't know whether the execution of a particular action $a$ results in the occurrence of event $e_1$ or event $e_2$. In this case, of course the agent cannot point out an actual event among the two. So it makes sense to only consider non-pointed event models (or, equivalently, multi-pointed event models where all events are designated). Technically speaking, we could alternatively just have introduced the standard syntax and semantics and then introduced the notation $[\mathcal{E}]\phi$ as an abbreviation of $\bigwedge_{e \in E} [\mathcal{E},e] \phi$, for all $\mathcal{E} = (E,Q,pre,post)$ and all formulas $\phi$.


\begin{definition}[Product update]
Let $\mathcal{M} = (W,R,V)$ be an epistemic model over $P$ and $\mathcal{E} = (E,Q,pre,post)$ an event model over $P$. The \emph{product update} of $\mathcal{M}$ with $\mathcal{E}$ is the epistemic model $\mathcal{M} \otimes \mathcal{E} = (W',R',V')$, where
\begin{itemize}
  \item 
 $W' = \{ (w,e) \in W \times E ~|~ (\mathcal{M}, w) \models pre(e) \}$
  \item 
 $R' = \{ ((w,e),(v,f)) \in W' \times W' ~|~ wRv \text{ and } eQf \}$
  \item 
$V'((w,e)) = \{p  \in P  ~|~ \mathcal{M}, w \models post(e)(p)  \}$
  \end{itemize}
\end{definition}
The product update $\mathcal{M} \otimes \mathcal{E}$ represents the result of executing the action $\mathcal{E}$ in the epistemic model represented by $\mathcal{M}$. It is well known that for finite models, two epistemic models are modally equivalent (satisfy the same formulas) iff they are bisimilar~\cite{blackburn2001modal}. For single-agent models as we consider here, bisimilarity of two models $\mathcal{M}$ and $\mathcal{M}'$ reduces to checking whether for each connected component of $\mathcal{M}$ (each equivalence class of worlds wrt $R$ in $\mathcal{M}$) there exists a connected component of $\mathcal{M}'$ containing the same valuations, and vice versa. Each epistemic model can easily be replaced by its minimal bisimilar representation (its \emph{bisimulation contraction}) achieved by
only preserving one world for each set of worlds with identical valuations within a connected component.
In the following, we will systematically assume each epistemic model achieved through a product update to be replaced by its bisimulation contraction, and we will generally identify isomorphic models. In this way, we consider bisimilar models to be identical. 
\begin{example}\label{exam:hospital2}
Continuing Example~\ref{exam:hospital1}, consider a situation of an agent seeing a coin lying heads up. It can be represented by the epistemic model $\mathcal{M} = (\{ w\} , \{(w,w) \} ,V)$ with $V(w) = \{ h \}$. Let us now calculate the result of executing the coin toss in this model;
\[
  \mathcal{M} \otimes \mathcal{E}\ =  \quad  
\raisebox{-2mm}{
  \begin{tikzpicture}[>=latex,every loop/.style={->},minimum size=1.5mm,every node/.style={auto}]
    \node[label={below:$(w_1,e_1)\!:h$},w] (w1) at (0,0) {}; 
    \node[label={below:$(w_1,e_2)\!:\text{ }$},w] (w2) at (3,0) {}; 
\end{tikzpicture} 
}
\]
Here, each world is labelled by its valuation, i.e.~the atomic propositions true at the world. In $\mathcal{M}$, the agent knows that the coin is facing heads up, that tossing it is neither guaranteed to lead to heads nor tails, but after the coin has been tossed, the agent will know. Those facts are encoded by the following:
\[
  \mathcal{M}, w \models K ( h \wedge \neg  [\mathcal{E}] h \wedge \neg  [\mathcal{E}] \neg h \wedge  [\mathcal{E}] (K h \vee K \neg h) 
  )
\]  
\end{example}

In this paper, we only consider single-agent DEL, as we consider a single agent trying to learn the dynamics of an environment. It might at first seem excessive to introduce all the machinery of DEL to only consider the single-agent version. It is well known that for single-agent epistemic logic, it is sufficient to represent epistemic models as sets of propositional states (often called \emph{belief states}). However, framing our results in the general setting of DEL is relevant for (at least) four reasons. The first is that it gives us the ability to use the event models of DEL to provide compact representations of actions (more compact than just describing actions as a set of possible transitions). Of course, other languages for action descriptions exist, in particular languages like STRIPS \cite{fikes1971strips} for describing planning domains, but those languages rarely support compact representations of partially observable actions. The second reason is that DEL integrates dynamic modalities in the logical language, so that our agents/learners can explicitly formulate their knowledge about action consequences, as was illustrated in Example~\ref{exam:hospital2}. The third is that the ultimate goal of the line of research introduced here is to be able to generalise to the multi-agent case, where a learner might learn not only what an action does and from what the agent observes, but also from what other agents observe. The fourth reason is that we also intend to integrate our learning algorithms into epistemic planning robots based on DEL~\cite{dissing2020implementing}. The goal is that the robots can not only do planning based on known actions, but can also learn new actions and new environment dynamics.


\section{Transition systems and partially observable domains \label{sec:domains}}
A domain typically consists of a set of states, a set of actions, and a state transition function mapping pairs of states and actions into the possible successor states~\cite{geffner2013concise}. When domains are partially observable, we also need to specify what is observed, e.g., with the use of an observation function. A domain can be deterministic or non-deterministic depending on whether an action can have one or more possible outcomes (one or more possible successor states). In this paper, we are only concerned with deterministic actions. However, a learner might during learning consider several possible outcomes of a given action, and hence we also consider non-deterministic transition systems. The set of states of a domain are typically specified as a subset of $2^P$ for some finite set of propositional symbols $P$. In this paper, our learners only try to learn domains of this type, but they might still represent their knowledge of such domains by domains of a more general type. Hence our definition of a domain will be more general. 

Consider a domain where the states are subsets of $P$ for some set of propositional symbols $P$. Observation functions often map into so-called `observation tokens' that can be completely separate from the language $P$ used to describe states. In our setting, however, we will assume that what is being observed in a state is the truth-value of a subset of the propositional symbols, hence directly connecting observations to the state descriptions themselves. For instance, if $l_r$ is a propositional symbol denoting that the light is on in room $r$, and $s$ is a state in which an agent is present in room $r$, the agent would be observing the truth-value of $l_r$ in state $s$. Given these assumptions, we now first define (labelled) transition systems~\cite[Ch~1]{sangiorgi_2011} and then our partially observable domains.
\begin{definition} A \emph{transition system} 
is a tuple $\mathcal{T} = (S,\actions,\transmod,s_0)$ where 
\begin{itemize}
    \item $S$ is a finite set of \emph{states}  
    \item $\actions$ is a finite set of \emph{actions}
    \item $\transmod: S \times \actions \rightarrow 2^S$ is a \emph{transition function}
    \item $s_0 \in S$ is the \emph{initial state}
\end{itemize}
 A transition system is called \emph{deterministic} if for all $s \in S$ and $a \in A$, $| T(s,a) | \leq 1$. An action $a \in A$ is called \emph{universally applicable} if for all $s \in S$, $| T(s,a) | \geq 1$.
 In a deterministic transition system where all actions are universally applicable, we hence have that $T(s,a)$ is a singleton for all $s$ and $a$.
   In that case, we often write $T(s,a) = s'$ instead of $T(s,a) = \{s'\}$, that is, we take $T$ to be a mapping $T: S \times A \to S$.  A transition system over a finite set of propositional symbols $P$ is a transition system where $S \subseteq 2^P$.
 \end{definition}
In this paper, we are concerned with learning actions of transition systems that are deterministic and in which every action is universally applicable. It might not be natural for any action to be applicable in any state, e.g.,\ ``open door'' might not be applicable in a state where the door is locked. However, we can replace any such action $a$ by a ``try $a$'' action $a'$ that is universally applicable by simply letting $T(s,a') = s$ for the states $s$ in which $a$ is not applicable (we can always attempt to open the door even if locked, but then it will simply stay closed). 
 We will also restrict attention to transition systems that are `generated' by their initial state $s_0$, that is, where any state $s \in S$ can be reached from $s_0$ by some action sequence. This simplifies things by making it clear that the learner always starts in the same state, and has the possibility to reach any state of the system.  
\begin{definition}   
   A \emph{(partially observable) domain}  is a tuple $\mathcal{D} = (\mathcal{T},\obs,\obsmod)$ 
    where
    \begin{itemize}
      \item  $\mathcal{T} = (S,A,T,s_0)$ is a transition system in which every action is universally applicable and in which every state can be reached by some action sequence applied to $s_0$
      \item $\obs$ is a set of \emph{observations}
      \item $\obsmod: S \to 2^\Omega$ is an \emph{observation function} mapping each state $s$ into the set of observations that are possible to receive in $s$
  \end{itemize}
       An observation function $\obsmod$ is \emph{deterministic} if $| \obsmod(s) | = 1$.        
       In that case we often write $\obsmod(s) = o$ instead of $\obsmod(s) = \{ o \}$, that is, we take $\obsmod$ to be a mapping $S \to \Omega$. In this case, an agent entering state $s$ will always receive the same observation $\obsmod(s)$. A domain $\mathcal{D} = ((S,A,T,s_0),\obs,\obsmod)$ is \emph{deterministic} if both $T$ and $\obsmod$ are deterministic. In this paper, we are going to assume that the observation function is deterministic.   
   
A domain \emph{over a set of propositional symbols} $P$ is $\mathcal{D} = ((S,A,T,s_0),\obs,\obsmod)$ where $S \subseteq 2^P$ and $\obs =  2^P \times 2^P$. An observation $(o^+,o^-) \in Obs(s)$ is split into a set $o^+$ of the \emph{propositions observed to be true} and $o^-$ of \emph{propositions observed to be false}.  We will assume observation functions to be \emph{noiseless}, that is, every proposition in $o^+$ is true in $s$, and every proposition in $o^-$ is false in $s$. Mostly, our domains over $P$ will be deterministic, and any state $s$ then always produces the same observation $Obs(s) = (o^+,o^-)$. 
\end{definition}
When $\mathcal{D} = ((S,A,T,s_0),\Omega,Obs)$, we will sometimes use $s \in \mathcal{D}$ as an abbreviation for $s \in S$. 

Even though we will only consider learning deterministic domains, such domains could still appear non-deterministic to a learner. For instance, an action $a$ might produce distinct outcomes when applied in two states $s$ and $t$ that are \emph{observationally indistinguishable}, that is, for which $Obs(s) = Obs(t)$. In such cases, a learner might (provisionally) decide to represent what has been learned as a non-deterministic transition system, even if the learner actually knows that the underlying transition system must be deterministic. 

In the following, we will identify a transition function $T: S \times \actions \to 2^S$ with its induced relation on $S \times \actions \times S$, given by: $(s,a,s') \in T$ iff $s' \in T(s,a)$. 
 Given a state $s$ in a transition system over $P$, we define $s^+ := s$ and $s^- := P-s$. In other words, $s^+$ is the set of atomic propositions true in $s$, and $s^-$ is the set of atomic propositions false in $s$.  Given a domain $\mathcal{D} = ((S,A,T,s_0),\obs,\obsmod)$, we define the set of \emph{possible observations} in $\mathcal{D}$ to be $\{ o \in \obsmod(s) \mid s \in S \}$. For deterministic domains over $P$, the set of possible observations is $\{ (x,y) \in 2^P \times 2^P \mid Obs(s) = (x,y) \text{ for some } s \in S \}$. For a possible observation $o = (x,y)$, we often use $o^+$ to denote the set of positively observed propositions $x$ and $o^-$ the set of negatively observed propositions $y$. Similarly, for an observation $Obs(s) = (x,y)$, we often write $Obs^+(s)$ for $x$ and $Obs^-(s)$ for $y$.

  To improve readability, we will often represent states by the sequence of literals true in the state, so e.g.\ if $P = \{p,q,r\}$, the state $s = \{p,r\}$ will be represented as $p\neg q r$. More generally, for any sequence of propositional literals $l_1,\dots,l_n$ containing exactly one occurrence of each propositional symbol in $P$, the state $s$ satisfying all the $l_i$ is denoted $l_1\cdots~l_n$. When the sequence $l_1,\dots,l_n$ contains \emph{at most} one occurrence of each propositional symbol in $P$, we use the term \emph{$l_1\cdots~l_n$-state} to denote any state $s$ satisfying all the $l_i$, i.e.,~any state $s$ with $s \models l_1 \wedge \cdots \wedge l_n$. So for instance if $P = \{p,q,r\}$, then the $pq$-states are $pqr$ and $pq\neg r$.  
  We will also sometimes use this type of notation for observations: Given the observation $(Obs^+(s),Obs^-(s))$ received in $s$, if $Obs^+(s) = \{ p_1,\dots,p_n \}$ and $Obs^-(s) = \{ q_1,\dots, q_m \}$, we will represent the observation compactly as $p_1 \cdots p_n \neg q_1 \cdots \neg q_n$. 

 \newcommand{\rooms}[3]{
   \begin{tikzpicture}[scale=0.6]
     \draw (0,0) rectangle (1.5,1);
 \draw (1.5,1) rectangle (3,0); 
     \ifthenelse{\equal{#1}{0}}{\node[scale=1.7] at (2.55,0.5) {\faLightbulbO};}{\node at (2.55,0.5) {\includegraphics[height=0.6cm]{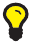}};}
    \ifthenelse{\equal{#3}{0}}{\node[rotate=90,scale=1.4] at (0.4,0.5) {\small \faToggleOff};}{\node[rotate=90,scale=1.4] at (0.4,0.5) {\small \faToggleOn};}
     \ifthenelse{\equal{#2}{0}}{\node[scale=1.4] at (1.1,0.5) {\faOdnoklassniki};}{\node[scale=1.4] at (1.9,0.5) { \faOdnoklassniki}; }
 \node at (-0.2,0) {};
 \node at (3.0,0) {};
   \end{tikzpicture}
 }
\begin{example}\label{example:lightswitch_domain}
Consider an environment with two rooms, the left room containing a toggle switch and the right room containing a lamp, see Figure~\ref{figu:lightswitch_transsys}. The switch controls the lamp, but the status of the light in the right room cannot be observed when being in the left room---and the status of the switch cannot be observed when being in the right room. We can model the environment as a domain over $P = \{l, r, s\}$, where $l$ means that the light is on, $r$ means that the agent is in the right room, and $s$ means that the switch is on. The initial state is $s_0 = \neg l \neg r \neg s$: the light and the switch are both off and the agent is in the left room. There are two actions available to the agent, $\it{flip}$ and $\it{move}$. Executing the $\it{flip}$ action flips the switch, that is, flips the truth value of both $s$ and $l$. Executing the $\it{move}$ action means moving to the other room, that is, flipping the truth value of $r$. Hence $A = \{\it{flip}, \it{move} \}$, and the transition system $(S,A,T,s_0)$ underlying the domain is then the system illustrated in Figure~\ref{figu:lightswitch_transsys}. We have no outgoing edges for the $\it{flip}$ action in $s_2$ and $s_3$, since the switch cannot be operated when being in the right room. To get a transition system in which every action is universally applicable, we simply assume there to be a reflexive loop for the $\it{flip}$ action in $s_2$ and $s_3$ (the $\it{flip}$ action is replaced by a ``try $\it{flip}$'' action). We often leave such reflexive edges implicit. 
 \begin{figure}
 \[
 \begin{tikzpicture}[auto,align=center]
   \node 
   (s0) {\rooms{0}{0}{0} \\[-1mm]  $s_0: \neg l \underline{ \neg r \neg s}$};
   \node[below left of=s0,node distance=30mm] (s1) {
   \rooms{1}{0}{1} \\[-1mm] $s_1: l \underline{\neg r s}$};
   \path[latex-latex] (s0) edge node[above left] {$\it{flip}$} (s1);
      \node[below of=s0,node distance=42mm] (s2) {\rooms{1}{1}{1} \\[-1mm] $s_2: \underline{l r} s$};
   \path[latex-latex] (s1) edge node[below left] {$\it{move}$} (s2);
   \node[below right of=s0,node distance=30mm] (s3) {\rooms{0}{1}{0} \\[-1mm] $s_3: \underline{\neg l r} \neg s$};
    \path[latex-latex] (s0) edge node {$\it{move}$} (s3);
    \node[right of=s3,node distance=40mm] {\begin{tabular}{|c|l|} \hline
    \raisebox{-0.5mm}{\scalebox{1.2}{\faLightbulbO}} & light bulb off \\
    \raisebox{-1mm}{\includegraphics[height=0.4cm]{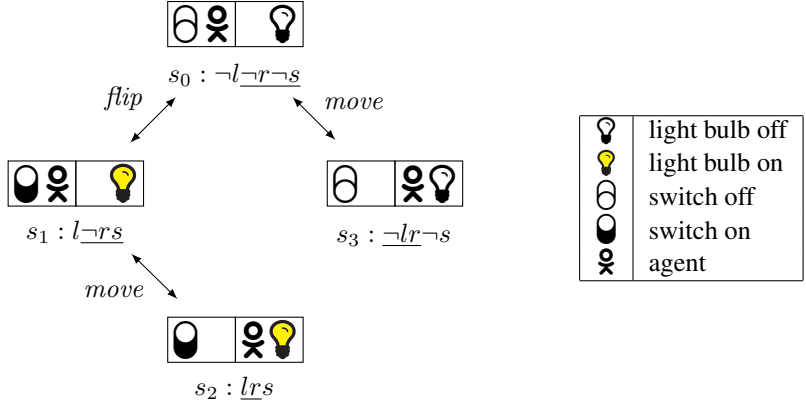}} & light bulb on \\
    \rotatebox[origin=c]{90}{\faToggleOff} & switch off \\
    \rotatebox[origin=c]{90}{\faToggleOn} & switch on \\
    \faOdnoklassniki & agent \\ \hline
    \end{tabular} 
    }; 
 \end{tikzpicture}
 \]
\caption{The transition system for the light switch domain.}\label{figu:lightswitch_transsys}
 \end{figure}
Note that the transition system is deterministic. Letting $S = \{s_0,s_1,s_2,s_3\}$, it is also a transition system in which each state is reachable from the initial state $s_0$. 

The environment can now be described as the domain $\mathcal{D} = ((S,A,T,s_0),\obs,\obsmod)$ over $P$ with $Obs(s_0) = \neg r \neg s$, $Obs(s_1) = \neg r s$, $Obs(s_2) = l r$, and $Obs(s_3) = \neg l r$. Note that we could have equivalently  defined $Obs$ by saying that the truth value of $r$ (location of agent) is observed in any state, that the truth value of $s$ (the position of the switch) is observed only in the left room, and that the truth value of $l$ (the status of the lamp) is only observed in the right room. In Figure~\ref{figu:lightswitch_transsys}, we have underlined the observed literals in each state.
\end{example}

\section{Learning explicit domain knowledge\label{sec:explicit}}
When learning fully observable domains, the learning goal is to learn the full underlying transition system (what each action does). This is possible under some reasonable assumptions~\cite{walsh2008efficient}. However, for partially observable domains, the goal of learning the full underlying transition system will in general not be attainable. An agent that can in no state observe the truth value of $p$, will never learn how its actions affect $p$. So we have to revise the learning goal to be that the learner learns whatever is possible given its observational limitations. The challenge is then generalising the learning from the case of full observability to the case of partial observability. 
Indeed, even the earlier claim---that an agent that can never observe the truth value of $p$, can never learn how its actions affect $p$---is not always true. 
\begin{example}\label{example:qq}
Consider a domain over two propositions $p$ and $q$. There is a single action $a$. The truth value of $q$ is always observed and the truth value of $p$ is never observed. Suppose that whenever a sequence of executions of $a$ is performed in the initial state, the learner receives the following sequence of observations: $q, q, \neg q, q, q, \neg q, \dots$, that is, a sequence of two $q$-states followed by a $\neg q$-state, then two $q$-states again, etc. Hence, in some of the $q$-states, the action $a$ produces another $q$-state, and in other $q$-states it produces a $\neg q$-state. 
Since the underlying transition system is assumed to be deterministic, then if we assume our learners to know this (which we are generally going to do), a learner receiving such observation sequences should hence be able to infer that there must be two distinct kinds of $q$-worlds. Assume  furthermore that our learners know the underlying language of the domain, that is, the available propositional symbols (another assumption that we are generally going to make in this paper). Then a learner can infer that the two distinct $q$-worlds must be distinct by assigning different truth values to $p$. In other words, the learner can infer that the $a$ action affects $p$. The learner can never know exactly \emph{how}, i.e.~will never be able in any state to infer the real truth value of $p$. But the learner can still learn something about the relation between the action $a$ and the proposition $p$, e.g.\ that $p$ does not always have the same truth value, but will have the same truth value every third time $a$ is executed.  \end{example}

We can distinguish between two types of knowledge that the learner can achieve, \emph{implicit} and \emph{explicit} knowledge. There is a rich literature in epistemic logic on explicit and implicit knowledge, and with different meanings assigned to the two concepts. In this paper, we will take explicit knowledge to be what is known because it is directly observed in the current state, and implicit knowledge to be whatever might additionally be inferred from the history of earlier actions and general experience with the environment (domain). In Example \ref{example:qq} the only explicit knowledge the learner can gain is that action $a$ sometimes makes a $q$-state into another $q$-state and sometimes into a $\neg q$-state; and that it always makes a $\neg q$-state into a $q$-state. Learning that every second time $a$ is executed in a $q$-state it produces a $\neg q$-state is not something we will consider to be explicit knowledge, as it can never be explicitly known how many times $a$ has been executed before  (it is not directly observed in any state). 

At first, we will only focus on learning explicit knowledge, that is, learners that learn to know what will be directly observed when actions are executed. Even if an agent is able to learn more than this, it is still interesting to have learners that can identify the explicit knowledge resulting from action executions. Consider for instance a person, Agnes, in a room with a switch that controls the light in the room next door, similar to the scenario of Example~\ref{example:lightswitch_domain}. Suppose the door to the other room is closed, so that it is not possible to observe whether the light in there is on or not. Suppose further that initially Agnes doesn't know what the light switch does, since she is in a house that she hasn't been in before, a summer house that she borrowed from a friend. If Agnes decides to explore the environment (the summer house), she might learn that the light switch controls the light in the other room, e.g.,\ by opening the door or peeking through the key hole. Hence, potentially, she might in this case be able to learn the full underlying transition system of the domain consisting of the light switch and the lamp in the other room. In this case, her implicit knowledge about the domain will be complete. However, identifying the implicit knowledge and the full underlying transition system is not the only thing that's relevant. Consider for instance that after the end of her stay, another friend, Bertram, comes to stay in the summer house. Bertram also never stayed there before. When Bertram arrives, the door to the other room might be closed, and the light in there is on. In this case, it would clearly be relevant for Agnes to know that Bertram now is in the same information state as she was initially, and then he will not be able to observe the light in the other room. So she might tell him: ``The light in the other room is on. This switch controls the light. Please make sure to turn off the light before you leave.'' In other words, at least in multi-agent scenarios, it is relevant to learn not only as much as we can possibly come to know about the environment, but also what we are able to directly see and not see in this environment.

\subsection{Compatibility domain}     
To define a learner for explicit knowledge, we first need some additional technical definitions. Given a deterministic domain $\mathcal{D} = ((S,A,T,s_0),\obs, \obsmod)$ over some set $P$, suppose an execution of an action $a \in A$ produces a transition $(s,t)$ (that is, we have $T(s,a) = t$). Then the learner observing the action execution doesn't get to see the state transition from $s$ to $t$ itself, but only observes a transition from the observation $(Obs^+(s),Obs^-(s))$ to the observation $(Obs^+(t),Obs^-(t))$. 
Generally, given a true state $s$, the learner observes only $(Obs^+(s),Obs^-(s))$. What is the set of states $t$ that the learner thinks could have produced the observation? If the learner has no additional information about the domain, it can only conclude that the underlying state $t$ must be among the states consistent with what has been observed, that is, a state $t$ satisfying $Obs^+(s) \subseteq t^+$  and  $Obs^-(s) \subseteq t^-$. More formally, given a state $s$, we define the states \emph{observationally compatible} with $s$ as the set of states $comp(s) := \{ t \in 2^P \mid Obs^+(s) \subseteq t^+  \text{ and } Obs^-(s) \subseteq t^- \}$. We can extend this notion to observations $(x,y) \in 2^P \times 2^P$ in the obvious way: $comp((x,y)) = \{ t \in 2^P \mid x \subseteq t^+  \text{ and } y \subseteq t^- \}$. We then get that $comp(s) = comp((Obs^+(s),Obs^-(s))$. Note that for any states $s,t \in S$, we have $comp(s) = comp(t)$ iff $Obs(s) = Obs(t)$. 
\begin{definition}
Let $\mathcal{D} = ((S,A,T,s_0),\obs,\obsmod)$ be a deterministic domain over $P$. The \emph{compatibility domain} $((S',A',T',s'_0),\obs',\obsmod')$ induced by $\mathcal{D}$ is given by:
\begin{itemize}
  \item $S' = \{ comp(s) \mid s \in S \}$
  \item $A' = A$
  \item $T' = \{ (comp(s),a,comp(t)) \mid (s,a,t) \in T \}$ 
  \item $s'_0 = comp(s_0)$ 
  \item $\Omega' = 2^P \times 2^P$
  \item $Obs'(comp(s)) = Obs(s)$.\footnote{This is well-defined since $comp(s) = comp(t)$ iff $Obs(s) = Obs(t)$.} 
\end{itemize}
\end{definition}
The compatibility domain $\mathcal{D'}$ induced by a domain $\mathcal{D}$ is the image of $\mathcal{D}$ under the compatibility mapping $comp$. It hence encodes the original transition systems as seen through the lens of the observation function. It encodes what is directly observed by the agent when actions are executed in the domain. Note that $\mathcal{D}'$ is not necessarily deterministic even if $\mathcal{D}$ is: a compatibility state $comp(s)$ might contain distinct states for which a given action $a$ produces distinct outcomes, and where these distinct outcomes can be observationally distinguished.  

 \begin{example}
 \begin{figure}
 \[
 \scalebox{0.9}{
 \begin{tikzpicture}[auto,align=center]
  \node[circle,draw,inner sep=-13pt,label={left:$comp(s_0)$ \\ $Obs = \neg r \neg s$}] (s0) at (0,0) {
 \begin{tikzpicture}[auto,align=center,inner sep=0pt] 
   \node (s01) {\rooms{0}{0}{0} \\[-1mm]  $s_0: \neg l \underline{ \neg r \neg s}$};
   \node[below of=s01,node distance=13mm] (s02) {\rooms{1}{0}{0} \\[-1mm]  $s'_0: l \underline{ \neg r \neg s}$}; 
  \end{tikzpicture}
  };
  \node[below left of=s0,node distance=40mm,circle,draw,inner sep=-13pt,label={left:$comp(s_1)$ \\ $Obs = \neg r s$}] (s1) {
  \begin{tikzpicture}[auto,inner sep=0pt]
  \node (s11) {\rooms{1}{0}{1} \\[-1mm] $s_1:  l \underline{\neg r s}$};
  \node[below of=s11,node distance=13mm] (s12) {\rooms{0}{0}{1} \\[-1mm] $s'_1: \neg l \underline{\neg r s}$};
  \end{tikzpicture}
  };
   \path[latex-latex] (s0) edge node[above left] {$\it{flip}$} (s1);
      \node[below of=s0,node distance=58mm,circle, draw, inner sep=-13pt,label={left:$comp(s_2)$ \\ $Obs = lr$}] (s2) {
      \begin{tikzpicture}[inner sep=0pt]
      \node (s21) {\rooms{1}{1}{1} \\[-1mm] $s_2: \underline{l r} s$};
      \node[below of=s21,node distance=13mm] (s22) {\rooms{1}{1}{0} \\[-1mm] $s'_2: \underline{l r} \neg s$};
      \end{tikzpicture}
      };
   \path[latex-latex] (s1) edge node {$\it{move}$} (s2);
    \node[below right of=s0,node distance=40mm,circle,draw,inner sep=-13pt,label={left:$comp(s_3)$ \\ $Obs=\neg l r$}] (s3) {
    \begin{tikzpicture}[inner sep=0pt]
       \node (s21) {\rooms{0}{1}{0} \\[-1mm] $s_3: \underline{\neg l r} \neg s$};
       \node[below of=s21,node distance=12mm] (s22) {\rooms{0}{1}{1} \\[-1mm] $s'_3: \underline{\neg l r} s$};
    \end{tikzpicture}
    };
    \path[latex-latex] (s0) edge node {$\it{move}$} (s3);
 \end{tikzpicture}
 }
 \]
  \caption{The compatibility domain $\mathcal{D}'$ induced by the light switch domain $\mathcal{D}$ of Example~\ref{example:lightswitch_domain}. } \label{figu:lightswitch_compdomain}
  \end{figure}
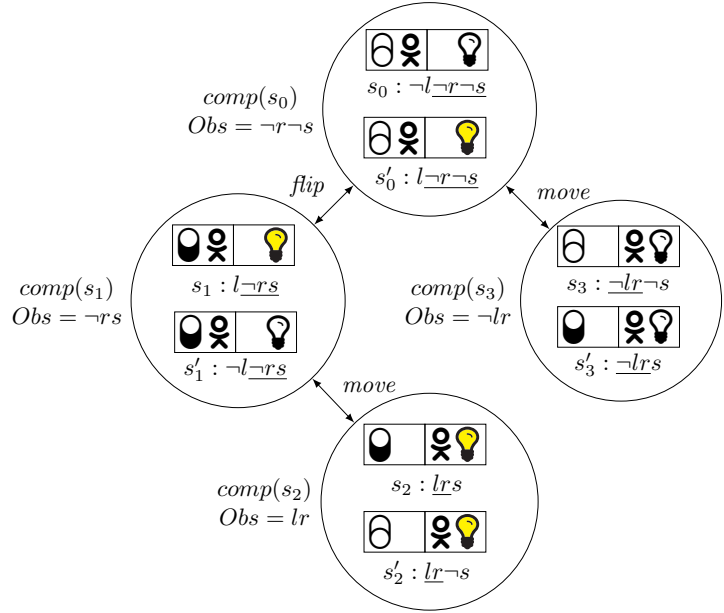
 Consider again the light switch domain $\mathcal{D}$ of Example~\ref{example:lightswitch_domain}. Figure~\ref{figu:lightswitch_compdomain} shows the compatibility domain $\mathcal{D}'$ induced by $\mathcal{D}$. In the figure, each state is marked by which compatibility state, $comp(s_i)$, it is, and additionally by the observation that defines it (the observation received in $s_i$). Note that $\mathcal{D}'$ is still deterministic. In fact, the underlying transition system of $\mathcal{D}$ is isomorphic to the underlying transition system of $\mathcal{D}'$. 
   At first one might think that this implies that a learner can learn to identify the full underlying system, despite its observational limitations. This, however, is not so. Even if a learner can use its observational powers to construct $\mathcal{D}'$ from observing action executions, there are things it can never learn. For instance, while the learner will be able to infer that, from the initial state, first flipping the switch and then moving to the other room will lead to a state where the light is on (the compatibility state $comp(s_2) = \{ s_2, s'_2 \}$), it will never be able to know whether it is the action of flipping the switch itself, or only the combined action of first flipping the switch and then moving, that made the light go on. This is due to the effect of the $\it{flip}$ action applied to $comp(s_0)$ being $comp(s_1) = \{s_1,s'_1\}$, in which there is both an $l$-state and a $\neg l$-state. In other words, the agent will never be able to distinguish the real domain from one in which the electrical circuit controlling the light has a serial connection containing both the light switch and a movement detector, 
   so that the light only goes on when both the switch is on and the agent is in that room. Since this distinction is not learnable, we of course need a notion of \emph{learnability} or \emph{behavioural correctness} that takes these observational limitations into account, so that a learning agent can still be considered ``complete'' even if it doesn't manage to identify the underlying transition system, but only identifies it up to its observational limitations.  
 \end{example}
 
 \begin{definition}
An \emph{isomorphism} between two domains $\mathcal{D} = ((S,A,T,s_0),\obs,\obsmod)$ and $\mathcal{D}' = ((S',A,T',s_0),\obs,\obsmod')$ is a bijection $f: S \to S'$ satisfying:
\begin{enumerate}
    \item $f(s_0) = s'_0$
    \item for any $s,t \in S$, $(s,a,t) \in T$ iff $(f(s),a,f(t)) \in T'$
    \item for all $s \in S$, $Obs(s) = Obs'(f(s))$.
\end{enumerate}
Two domains are called \emph{isomorphic} if there exists an isomorphism between them. 
\end{definition}
This is the natural generalisation of isomorphisms between labelled transition systems~\cite{gorrieri2015introduction} to domains. 
Let $\mathcal{D}'$ denote the compatibility domain of a domain $\mathcal{D} = ((S,A,T,s_0),\obs,\obsmod)$.
An alternative definition of the compatibility domain induced by $\mathcal{D}$ could be the domain $\mathcal{D}'' = ((S'',A,T'',s''_0),\obs,\obsmod'')$ given by $S' = \{ Obs(s) \mid s \in S \}$, $T' = \{(Obs(s),a,Obs(t)) \mid (s,a,t) \in T \}$, $s''_0 = Obs(s_0)$, and $Obs''(s) = s$. Since we have, for all $s,t \in S$, $comp(s)=comp(t)$ iff $Obs(s)=Obs(t)$, we can define a bijection $f: S' \to S''$ by $f(comp(s)) = Obs(s)$. It is easily verified that $f$ is an isomorphism between $\mathcal{D}'$ and $\mathcal{D}''$, so the two definitions of compatibility domains are equivalent (up to isomorphism). It means that we can also think of the compatibility domain to be the domain of observations with the transitions on observations induced by the real underlying transitions in the obvious way.

In a state $s$ of a domain $\mathcal{D}$, we consider the explicit knowledge of the agent to be what is true in all of $comp(s)$, which is of course simply what follows logically from the observation made in $s$, since the following are equivalent for all propositional formulas $\phi$ (where $\models$ denotes the standard entailment relation in propositional logic),
\begin{itemize}
  \item 
  for all $t \in comp(s)$, $t \models \phi$
  \item for all  $t \text{ with } Obs^+(s) \subseteq t^+ \text{ and } Obs^-(s) \subseteq t^-, t \models \phi$ 
  \item  $\models (\bigwedge_{p \in Obs^+(s)} p \wedge \bigwedge_{p \in Obs^-(s)} \neg p) \to \phi$
\end{itemize}
To make the notion of explicit knowledge precise in an epistemic setting, we identify any set of states $S$ over a set of propositional symbols $P$ with an \emph{induced} epistemic model $\mathcal{M}_S = (W,R,V)$ over $P$ with $W = S$, $R = W \times W$ and $V(w) = w$. By this identification, we can evaluate (static) epistemic formulas in sets of states, e.g.\ we then have $comp(s) \models K \phi$ iff $\phi$ is true in every world (state) $w$ in $comp(s)$, which again holds iff $\phi$ follows from the observation received in $s$. The modal operator $K$ is here interpreted as the operator for explicit knowledge. In the following, we will generally identify sets of states $S$ with their induced epistemic models $\mathcal{M}_S$ without notice. It will be clear from the context whether a set of states should be considered as just that, or as the induced epistemic model. 

By the identification of sets of states with epistemic models, we can see the compatibility domain as a domain on epistemic models. In dynamic epistemic logic, any set of event models similarly defines a domain on epistemic models.
\begin{definition}\label{def:domain_induced_by_events}
Let $(\mathcal{E}_a)_{a \in A}$ be a collection of event models over a set of propositions $P$, and let $\mathcal{M}_0$ be an epistemic model over $P$. The domain $((S',A',T',s'_0),\obs,\obsmod)$ induced by $(\mathcal{E}_a)_{a\in A}$ and $\mathcal{M}_0$ is given by: 
\begin{itemize}
  \item $S' = \{ \mathcal{M}_0 \otimes \mathcal{E}_{a_1} \otimes \cdots \otimes \mathcal{E}_{a_n} \mid a_1,\dots,a_n \in A \}$ 
  \item $A' = A$
  \item $T' = \{ (\mathcal{M},a,\mathcal{M}') \mid \mathcal{M'} \text{ is a connected component of } \mathcal{M} \otimes \mathcal{E}_a  \}$ 
  \item $s'_0 = \mathcal{M}_0$ 
  \item $\Omega = 2^P \times 2^P$
  \item $Obs(\mathcal{M}) = (\{ p \in P \mid \mathcal{M} \models p \}, \{ p \in P \mid \mathcal{M} \models \neg p \})$
\end{itemize}
\end{definition}

\subsection{Behavioural correctness and learners} 
We can use isomorphisms between domains to define (strong) equivalence between them. We could of course also consider a weaker notion of equivalence between domains given by bisimulation, and indeed we will later do so when defining learners for implicit knowledge. However, for now it suffices to consider isomorphisms. 
 
 If the domain induced by a collection of event models is isomorphic to the compatibility domain, it means that the event models represent exactly what is explicitly knowable about the domain. A learner that identifies such event models will be called \emph{behaviourally correct with respect to explicit knowledge}. It will be a learner that identifies the underlying transition system up to the limitations of its distinguishing powers defined by the compatibility mapping (defined by what is explicitly observable), as follows. 
\begin{definition}
Let $\mathcal{D} = ((S,A,T,s_0),\obs,\obsmod)$ and $\mathcal{D}' = ((S',A,T',s'_0),\obs',\obsmod')$ be domains. We say that $\mathcal{D}'$ is \emph{behaviourally correct with respect to explicit knowledge} about $\mathcal{D}$ if $\mathcal{D}'$ is isomorphic to the compatibility domain induced by $\mathcal{D}$. A collection of event models $(\mathcal{E}_a)_{a \in A}$ 
is \emph{behaviourally correct with respect to explicit knowledge} about $\mathcal{D}$ if the domain induced by $(\mathcal{E}_a)_{a \in A}$ and $comp(s_0)$ is behaviourally correct with respect to explicit knowledge about $\mathcal{D}$.
\end{definition}

Learners learn from observing the execution of actions. Executions of actions are represented as transitions $(s,a,t) \in T$. As mentioned earlier, when a transition $(s,a,t) \in T$ occurs, the learner only observes $(Obs(s),a,Obs(t))$. We call $(Obs(s),a,Obs(t))$ the \emph{observed transition} of $(s,a,t)$. In general, for a domain $\mathcal{D} = ((S,A,T,s_0),\obs,\obsmod)$, an \emph{observed transition} is any $(Obs(s),a,Obs(t))$ with $(s,a,t) \in T$. A learner takes as input a set of observed transitions of a domain, and attempts to provide a representation of what it has learned about the actions of the domain. There are many possible ways a learner could represent its learned actions, one of them being through event models. 
In that setting, a learner can be seen as an algorithm that takes as input a set $\sigma$ of observed transitions of actions $A$ over $P$ and produces a collection of event models $\mathcal{E}_a$, one for each action $a \in A$. The goal for the learner is then to produce a collection of event models that is behaviourally correct with respect to explicit knowledge. If the learner hasn't been exposed to all possible transitions of the system, it cannot of course in general be expected to produce behaviourally correct event models. If a set of observed transitions contains observations of all possible transitions of the system, we call it \emph{sound and complete}.  More precisely, a set of observed transitions $\sigma$ is called \emph{sound and complete} for a domain $\mathcal{D} = ((S,A,T,s_0),\obs,\obsmod)$ if for all $(s,a,t) \in T$, $(Obs(s),a,Obs(t)) \in \sigma$. 

\subsection{A behaviourally correct learner of explicit knowledge}
We now present our first learner (learning algorithm) and prove that it achieves its goal (producing a behaviourally correct set of event models) when presented with a sound and complete set of observed transitions. The algorithm is included as Algorithm \ref{algo:explicit_learner}. It relies on the following additional definition. For any observation $o$ in a domain over $P$, we define the epistemic formula $\phi_o$ representing the explicit knowledge that an agent observing $o$ has, where we use $\textit{Kw } \phi$ as shorthand for $K \phi \vee K \neg \phi$ (knowing whether $\phi$):
 \[ \textstyle 
 \phi_{o} \coloneqq  K \left(\bigwedge_{p \in o^+} p \wedge \textstyle\bigwedge_{p \in o^-} \neg p \right) \wedge \bigwedge_{p \in P - (o^+ \cup o^-)} \neg \textit{Kw } p \]
 Consider for instance the initial state $s_0 = \neg l \neg r \neg s$ of the light switch domain (Example~\ref{example:lightswitch_domain}). In this state, the agent receives the observation $Obs(s_0) = \neg r \neg s$ and hence $\phi_{Obs(s_0)}$ is the formula $K(\neg r \wedge \neg s) \wedge \neg \textit{Kw } l $ representing that the agent knows that the switch is off and that the agent itself is in the left room, but not knowing whether the light is on or off. 
  


\begin{algorithm}[t]
\SetAlgoLined
\Input{$P$ (propositional symbols), $A$ (actions), $\sigma$ (observed transitions)}
\Output{$(\mathcal{E}_a)_{a \in A}$ (event models)}
  \For{each $a \in A$ 
  }{
  let $\mathcal{E}_a$ be an empty event model\;
   \For{each observation $o$
   }{
  $O \coloneqq \{ o' \mid (o',a,o) \in \sigma \}$\;
  \If{$O \neq \emptyset$}{
   $\textstyle E \coloneqq \{\langle\bigvee_{o' \in O} \phi_{o'}, \post \rangle \mid \post(p) = \top \text{ for all } p \in o^+ \text{ and } \post(p) = \bot \text { for all } p \in o^- \}$\;
   Add to $\mathcal{E}_a$ the events $E$ and make all events in $E$ mutually indistinguishable, but distinguishable from all other events in $\mathcal{E}_a$\; 
}
 }
 }
 \Return{$(\mathcal{E}_a)_{a \in A}$}\;
 \caption{$\textsc{Learner}(P,A,\sigma)$}\label{algo:explicit_learner}
\end{algorithm}

\begin{theorem}
The learning algorithm $\textsc{Learner}(P,A, \sigma)$ (Algorithm~\ref{algo:explicit_learner}) applied to a sound and complete set of observed transitions of a domain $\mathcal{D}$ outputs a collection of event models that are behaviourally correct with respect to explicit knowledge about $\mathcal{D}$.  
\end{theorem} 
\begin{proof} 
Let $\sigma$ be a sound and complete set of observed transitions for a domain $\mathcal{D} = ((S,A,T,s_0),\obs,\obsmod)$, and let $(\mathcal{E}_a)_{a\in A}$ denote the output produced by Algorithm~\ref{algo:explicit_learner}. We need to show that the domain induced by $(\mathcal{E}_a)_{a\in A}$ and $comp(s_0)$ is isomorphic to the compatibility domain induced by $\mathcal{D}$. So let $\mathcal{D}' = ((S',A,T'),x_0),\obs,\obsmod')$ denote the compatibility domain of $\mathcal{D}$ and let $\mathcal{D}'' = ((S'',A,T'',\mathcal{M}_0),\obs,\obsmod'')$ denote the domain induced by $(\mathcal{E}_a)_{a\in A}$ and $\mathcal{M}_0 = comp(s_0)$, where $\obs = 2^P \times 2^P$. We need to find a bijection $f: S' \to S''$ such that $f(x_0) = \mathcal{M}_0$, for any $x,y \in S'$ we have $(x,a,y) \in T'$ iff $(f(x),a,f(y)) \in T''$, and for all $x \in S'$, $Obs'(x) = Obs''(f(x))$.

\medskip
\noindent \emph{Claim 1}. 
If $(x_i,a_i,x_{i+1}) \in T'$ for all $0 \leq i \leq n$, then $(\mathcal{M}_i, a_i, \mathcal{M}_{i+1}) \in T''$ for all $0 \leq i \leq n$ where $\mathcal{M}_i$ is the epistemic model induced by $x_i$. 

\smallskip
\noindent \emph{Proof of Claim 1.} 
The proof is by induction on $n$. Note that by definition, $\mathcal{M}_0$ is the epistemic model induced by the state $x_0 = comp(s_0)$, which covers the base case (action sequences of length 0). For the induction step, suppose $(x_i,a_i,x_{i+1}) \in T'$  for all $0 \leq i \leq n$ and suppose $(\mathcal{M}_i, a_i, \mathcal{M}_{i+1}) \in T''$ for all $0 \leq i \leq n-1$ where $\mathcal{M}_i$ for all $i \leq n$ is the epistemic model induced by $x_i$. We need to show $(\mathcal{M}_n,a_n,\mathcal{M}_{n+1}) \in T''$ where $\mathcal{M}_{n+1}$ is the epistemic model induced by $x_{n+1}$. 
 
By soundness and completeness of $\sigma$, $(Obs(s_n), a_n, Obs(s_{n+1})) \in \sigma$. This implies that $\mathcal{E}_{a_n}$ contains a subset $E$ of all events of the form $\langle \bigvee_{(o',a_n,Obs(s_{n+1})) \in \sigma} \phi_{o'}, \post \rangle$ where $\post(p) = \top$ for all $p \in Obs^+(s_{n+1})$ and $\post(p) = \bot$ for all $p \in Obs^-(s_{n+1})$. Furthermore, the events in $E$ are all mutually indistinguishable, but distinguishable from all other events of $\mathcal{E}_{a_n}$. Since $\mathcal{M}_n$ is induced by $x_n = comp(s_n)$, the valuations occurring in the worlds of $\mathcal{M}_n$ are exactly the ones represented by the states in $comp(s_n)$. That is, $\mathcal{M}_n$ has a world $w$ with valuation $V(w) = P'$ iff $Obs^+(s_n) \subseteq P'$ and $Obs^-(s_n) \cap P' = \emptyset$. From this it follows that $\mathcal{M}_n \models \phi_{Obs(s_n)}$. Since $(Obs(s_n), a_n, Obs(s_{n+1})) \in \sigma$, any world satisfying $\phi_{Obs(s_n)}$ will satisfy the precondition of each of the events in $E$. In particular, each world of $\mathcal{M}_n$ will satisfy the precondition of each event in $E$. 
For any state $t \in x_{n+1} = comp(s_{n+1})$ we must have $Obs^+(s_{n+1}) \subseteq t^+$ and $Obs^-(s_{n+1}) \subseteq t^-$. Hence for any state $t \in x_{n+1}$, there must exist an event $e_t \in E$ with $\post(e_t)(p) = \top$ for all $p \in t$ and $\post(e_t)(p) = \bot$ for all $p \not\in t$. Conversely, for any event $e \in E$, from the conditions  $\post(e)(p) = \top$ for all $p \in Obs^+(s_{n+1})$ and $\post(e)(p) = \bot$ for all $p \in Obs^-(s_{n+1})$ we get the existence of a state $t$ observationally compatible with $s_{n+1}$ such that $\post(e) = \post(e_t)$. This implies that $\mathcal{M}_n \otimes \mathcal{E}_{a_n}$ contains a connected component in which the world valuations are exactly the states $t \in x_{n+1}$. This proves that $\mathcal{M}_n \otimes \mathcal{E}_{a_n}$ contains $\mathcal{M}_{n+1}$ (the epistemic model induced by $x_{n+1}$) as a connected component, and hence $(\mathcal{M}_n,a_n,\mathcal{M}_{n+1}) \in T''$, as required (recall the models are identified under bisimilarity, so it is irrelevant that the connected component might contain several worlds with the same valuation). This completes the proof of the claim.

\medskip
\noindent \emph{Claim 2.} 
If $(\mathcal{M}_i,a_i,\mathcal{M}_{i+1}) \in T''$ for all $0 \leq i \leq n$, then there exists a state sequence $x_1,\dots,x_{n+1} \in S'$ such that for all $i$, $(x_i, a_i, x_{i+1}) \in T'$ 
and $\mathcal{M}_i$ is the epistemic model induced by $x_i$. 

\smallskip
\noindent \emph{Proof of Claim 2.} 
The base case is as for Claim 1. For the induction step, suppose $(\mathcal{M}_i,a_i,\mathcal{M}_{i+1}) \in T''$  for all $0 \leq i \leq n$ and suppose $(x_i, a_i, x_{i+1}) \in T'$ for all $0 \leq i \leq n-1$ where $\mathcal{M}_i$ for all $i \leq n$ is the epistemic model induced by $x_i$. We need to show $(x_n,a_n,x_{n+1}) \in T'$ and that $\mathcal{M}_{n+1}$ is the epistemic model induced by $x_{n+1}$. 
Since $(\mathcal{M}_n,a_n,\mathcal{M}_{n+1}) \in T''$, $\mathcal{M}_{n+1}$ is one of the connected components of $\mathcal{M}_n \otimes \mathcal{E}_{a_n}$. Since $\mathcal{M}_n$ is induced by $x_n$, it must be connected, and hence there must exist a connected component $E$ of events in $\mathcal{E}_{a_n}$ such that $\mathcal{M}_{n+1}$ is the product update of $\mathcal{M}_n$ with those events. By the definition of the algorithm, there must exist an observation $o$ such that $(o',a_n,o) \in \sigma$ for some $o'$ and such that $E$ must be the set of all events of the form $\langle \bigvee_{(o',a_n,o) \in \sigma} \phi_{o'}, \post \rangle$ where $\post(p) = \top$ for all $p \in o^+$ and $\post(p) = \bot$ for all $p \in o^-$. 
Since $\mathcal{M}_{n+1}$ is the product update of $\mathcal{M}_n$ with the events in $E$, and since $\mathcal{M}_{n+1}$ is non-empty, at least one of the events in $E$ must have its precondition satisfied in at least one of the worlds of $\mathcal{M}_n$. Since all events in $E$ have the same precondition, \emph{all} events have their precondition satisfied in at least one world of $\mathcal{M}_n$.  
In other words, there must exist $o'$ such that  
$(o',a_n,o) \in \sigma$ and such that $\phi_{o'}$ is satisfied in at least one of the worlds of $\mathcal{M}_n$. From $(o',a_n,o) \in \sigma$ we get the existence of $s_n, s_{n+1} \in S$ with $(s_n,a_n,s_{n+1}) \in T$, $Obs(s_n) = o'$ and $Obs(s_{n+1}) = o$. By the definition of $\phi_{o'}$ and since $\phi_{o'}$ holds in a world of $\mathcal{M}_n$, all propositions in $o'^+$ are true in all worlds of $\mathcal{M}_n$, all propositions in $o'^-$ are false in all worlds of $\mathcal{M}_n$, and all other propositions are true in some worlds of $\mathcal{M}_n$ and false in others. Since $\mathcal{M}_n$ is the epistemic model induced by $x_n$, we then get that $x_n = \{ t \in 2^P \mid o'^- \subseteq t^- \text{ and } o'^+ \subseteq t^- \} = \{ t \mid  Obs^-(s_n) \subseteq t^-  \text{ and } Obs^+(s_n) \subseteq t^+  \} = comp(s_n)$. Since $(s_n,a_n,s_{n+1}) \in T$, we must have $(comp(s_n),a_n,comp(s_{n+1})) \in T'$. It now suffices to prove that $\mathcal{M}_{n+1}$ is the epistemic model induced by $comp(s_{n+1})$. Since $Obs(s_{n+1}) = o$, $E$ is the set of all events of the form $\langle \bigvee_{(o',a_n,Obs(s_{n+1})) \in \sigma} \phi_{o'}, \post \rangle$ where $\post(p) = \top$ for all $p \in Obs^+(s_{n+1})$ and $\post(p) = \bot$ for all $p \in Obs^-(s_{n+1})$. This implies that the update of $\mathcal{M}_n$ with the events in $E$ must be all states in $comp(s_{n+1})$, and hence $\mathcal{M}_{n+1}$ is induced by $comp(s_{n+1})$, as required.  This completes the proof of the claim.

\medskip
We can now construct an isomorphism $f$ from $\mathcal{D}'$ to $\mathcal{D}''$ by, for all $x \in S'$, letting $f(x)$ be the epistemic model induced by $x$. Note that Claim 2 guarantees that all states of $S''$ are induced epistemic models of states in $S'$, hence guaranteeing that the mapping correctly maps elements of $S'$ into elements of in $S''$. We now first get $f(x_0) = \mathcal{M}_0$, as required. Suppose then $(x,a,y) \in T'$. Since each state $S$ is reachable from $s_0$, there must then exist an action sequence $a_0,\dots,a_n \in A$ and a state sequence $x_0,\dots,x_{n+1} \in S'$ such that for all $i\leq n$ we have $(x_i,a_i,x_{i+1}) \in T'$ and where $x=x_n$, $a=a_n$ and $y=x_{n+1}$. It follows from Claim 1 that $(\mathcal{M}_n,a_n,\mathcal{M}_{n+1}) \in T''$, where $\mathcal{M}_n$ is the epistemic model induced by $x_n$ and $\mathcal{M}_{n+1}$ is the epistemic model induced by $x_{n+1}$. We then get $f(x) = f(x_n) = \mathcal{M}_n$ and $f(y) = f(x_{n+1}) = \mathcal{M}_{n+1}$ and hence $(f(x),a,f(y)) \in T''$, as required. Suppose instead that $(f(x),a,f(y)) \in T''$. Then we need to prove that $(x,a,y) \in T'$. There must exist an action sequence $a_0,\dots,a_n \in A$ and a state sequence $\mathcal{M}_0,\dots,\mathcal{M}_{n+1} \in S''$ such that for all $i \leq n$ we have  $(\mathcal{M}_i,a_i,\mathcal{M}_{i+1}) \in T''$ and where $f(x)=\mathcal{M}_n$, $a=a_n$ and $f(y) = \mathcal{M}_{n+1}$. From Claim 2 we now get that $(x_n,a_n,x_{n+1}) \in T'$ where $\mathcal{M}_n$ is the epistemic model induced by $x_n$ and $\mathcal{M}_{n+1}$ is the epistemic model induced by $x_{n+1}$. It follows that $f(x_n) = f(x)$ and $f(x_{n+1}) = f(y)$, and hence $x_n = x$ and $x_{n+1} = y$, and thus finally $(x,a,y) \in T'$, as required. For the observation functions, it finally follows that for all $comp(s) \in S'$ we have $Obs'(comp(s)) = Obs''(f(comp(s)))$, since $f(comp(s))$ is the induced epistemic model of $comp(s)$, and hence $Obs''(f(comp(s))) = (\{ p \mid f(comp(s)) \models p \}, \{ p \mid f(comp(s)) \models \neg p \}) = (\{ p \mid p \in t \text{ for all } t \in comp(s) \}, \{ p \mid p \not\in t \text{ for all } t \in comp(s) \}) = (Obs^+(s),Obs^-(s)) = Obs(s) = Obs'(comp(s))$. 

\end{proof}

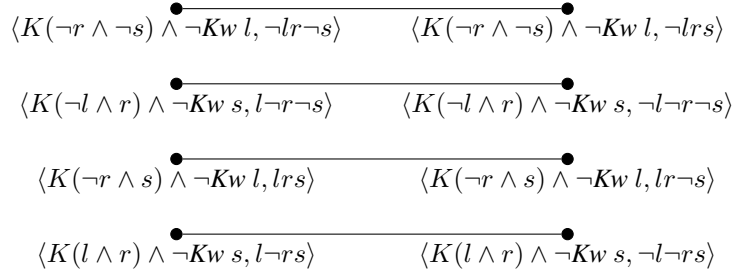
\begin{figure}
\centering
  \scalebox{1}{
    \begin{tikzpicture}[xscale=0.8,yscale=0.5,every label/.style={yshift=1mm}]
        \node[w, label={[align=right]below: $\langle K(\neg r \wedge \neg s) \wedge \neg \textit{Kw } l, \neg l r \neg s\rangle $}] (w0) at (0,0) {};
        \fill (w0) circle [radius=2pt];
        \node[w,  label={below: $\langle K(\neg r \wedge \neg s) \wedge \neg \textit{Kw } l, \neg l  r  s\rangle $}] (w1) at (6.5,0) {};
        \fill (w1) circle [radius=2pt];
      
        \node[w, label={[align=right]below: $\langle K(\neg l \wedge r) \wedge \neg \textit{Kw } s, l \neg r  \neg s\rangle $}] (w2) at (0,-2) {};
        \fill (w2) circle [radius=2pt];
        \node[w,  label={below: $\langle K(\neg l \wedge r) \wedge \neg \textit{Kw } s, \neg l \neg r \neg s\rangle $}] (w3) at (6.5,-2) {};
        \fill (w3) circle [radius=2pt];
        
        \node[w, label={[align=right]below: $\langle K(\neg r \wedge s) \wedge \neg \textit{Kw } l, l r  s\rangle $}] (w4) at (0,-4) {};
        \fill (w4) circle [radius=2pt];
        \node[w,  label={below: $\langle K(\neg r \wedge s) \wedge \neg \textit{Kw } l, l r  \neg s\rangle $}] (w5) at (6.5,-4) {};
        \fill (w5) circle [radius=2pt];
        
        \node[w, label={[align=right]below: $\langle K(l \wedge r) \wedge \neg \textit{Kw } s,l \neg r s \rangle  $}] (w6) at (0,-6) {};
        \fill (w6) circle [radius=2pt];
        \node[w,  label={below: $\langle K(l \wedge r) \wedge \neg \textit{Kw } s,\neg l  \neg r  s\rangle  $}] (w7) at (6.5,-6) {};
        \fill (w7) circle [radius=2pt];
        \draw[-] (w0) -- (w1) node[above,midway] {};
        \draw[-] (w2) -- (w3) node[above,midway] {};
        \draw[-] (w4) -- (w5) node[above,midway] {};
        \draw[-] (w6) -- (w7) node[above,midway] {};
    \end{tikzpicture}
    }
    \caption{The event model $\mathcal{E}_{\it{move}}$ provided as output of Algorithm 1 running on the light switch domain.}\label{figu:lightswitch_learned_move} 
\end{figure}{}
\begin{figure}
\centering
   \scalebox{1}{
    \begin{tikzpicture}[xscale=0.8,yscale=0.5,every label/.style={yshift=1mm}]
        \node[w, label={[align=right]below: $\langle K(\neg r \wedge \neg s) \wedge \neg \textit{Kw } l, l \neg r s \rangle $}] (w0) at (0,0) {};
        \fill (w0) circle [radius=2pt];
        \node[w,  label={below: $\langle K(\neg r \wedge \neg s) \wedge \neg \textit{Kw } l, \neg l \neg r  s \rangle $}] (w1) at (6.5,0) {};
        \fill (w1) circle [radius=2pt];
      
        \node[w, label={[align=right]below: $\langle K(\neg r \wedge s) \wedge \neg \textit{Kw } l, l \neg r \neg s \rangle $}] (w2) at (0,-2) {};
        \fill (w2) circle [radius=2pt];
        \node[w,  label={below: $\langle K(\neg r \wedge s) \wedge \neg \textit{Kw } l , \neg l \neg r \neg s \rangle $}] (w3) at (6.5,-2) {};
        \fill (w3) circle [radius=2pt];
        
        \draw[-] (w0) -- (w1) node[above,midway] {};
        \draw[-] (w2) -- (w3) node[above,midway] {};
    \end{tikzpicture}
    }
        \caption{The event model $\mathcal{E}_{\it{flip}}$ provided as output of Algorithm 1 running on the light switch domain.}\label{figu:lightswitch_learned_flip}
\end{figure}
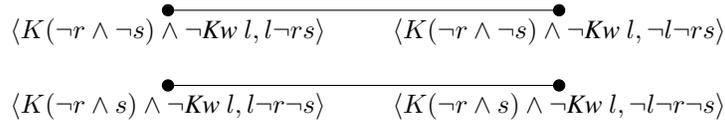{}
\begin{example}
Consider again the light switch domain $\mathcal{D}$ of Example~\ref{example:lightswitch_domain}. Letting $P = \{l,r,s\}$, $A = \{\it{move},\it{flip} \}$ and $\sigma$ be a sound and complete set of observed transitions for $\mathcal{D}$, $\textsc{Learner}(P,A,\sigma)$ will for the $\it{move}$ action produce the event model $\mathcal{E}_{\it{move}}$ presented in Figure~\ref{figu:lightswitch_learned_move} and for the $\it{flip}$ action produce the event model $\mathcal{E}_{\it{flip}}$ presented in Figure~\ref{figu:lightswitch_learned_flip}. Note that the event model $\mathcal{E}_{\it{move}}$ has four connected components, corresponding to whether the move goes from the  left to the right room (the first and the third connected components of Figure~\ref{figu:lightswitch_learned_move}) or from the right to the left room (the second and the fourth connected components), and whether the move starts in a state where the switch/light is off (first and second component) or on (third and fourth). In total, the two event models have 6 connected components corresponding to the 6 transitions of the compatibility domain induced by $\mathcal{D}$. Since each connected component has 2 events, this appears to be an even less compact representation than the compatibility transition system itself, but this is due to the event models representing action outcomes as being nondeterministic rather than representing them as deterministically mapping into a set of states (a ``belief state''). The induced compatibility domain of course represents the exact same information as the set of event models, only in a different way. It would probably be possible to represent the event models more compactly, e.g.\ using a three-valued version of DEL where having an event $e$ with $post(e)(p) = \textit{undefined}$ would be a compact representation of the two events $\langle pre(e), post(e) \cup \{ p \mapsto \top \} \rangle$ and $\langle pre(e), post(e) \cup \{ p \mapsto \bot \} \rangle$. Such a representation would reduce the event models of this example to only contain 6 events in total. We leave the detailed consideration of such three-valued versions of DEL to a future paper. Another possible optimisation of the action representation would be to use non-boolean postconditions, that is, postcondition mappings where $post(e)(p)$ can be an arbitrary formula. Actually, such postconditions are the most standard in DEL, however, for simplicity we decided here to only work with boolean postconditions. As an example, consider the second and fourth component of Figure~\ref{figu:lightswitch_learned_move}. Using non-boolean postconditions~\cite{ditmarsch2008semantic}, we could represent the two left events of the two components by a single event $\langle Kr \wedge \neg \textit{Kw } s, \post \rangle$ with $\post(l) = \top$, $\post(r) = \bot$ and $\post(s) = l$, and similarly for the two right events. Event preconditions can also sometimes be simplified, e.g.\ the aforementioned event could in this context be simplified to $\langle Kr, \post \rangle$. In future work it will be relevant to consider extended learning algorithms including methods for simplifying event models (ensuring that they still induce the same domain). 

The learned event models of Figures~\ref{figu:lightswitch_learned_move}--\ref{figu:lightswitch_learned_flip} represent the explicit knowledge the agent can gain about the \emph{dynamics} of the domain. Earlier, we explained how to evaluate static formulas in $comp(s_0)$, but given the learned event models, we can now also evaluate dynamic formulas including the modalities $[\mathcal{E}_{\it{flip}}]$ and $[\mathcal{E}_{\it{move}}]$. This allows the agent to reason about the dynamics of its explicit knowledge, for instance we get:
\[
  comp(s_0) \models \neg \textit{Kw } l \wedge [\mathcal{E}_{\it{flip}}]  \neg \textit{Kw } l \wedge [\mathcal{E}_{\it{flip}}] [\mathcal{E}_{\it{move}}] K l
\]   
Thus a learner having used Algorithm 1 to learn the dynamics of the domain and produced the event models of the two figures, would be able to conclude that it initially doesn't explicitly know whether the light is on, that flipping the switch doesn't change this, but both flipping the switch and moving to the other room guarantees explicitly knowledge of the light being on. This can potentially be used in combination with epistemic planning based on DEL~\cite{bolander2011epistemic}. For instance, if the learner was given the planning goal $K l$ (turning on the light and explicitly knowing to have done so), it would produce the plan $(\it{flip}, \it{move})$, since this is the shortest action sequence $(a_1,\dots,a_n)$ satisfying $comp(s) \models [a_1]\cdots [a_n] K l$. Note that if the learner was given the simpler goal $l$ (``turn on the light''), the plan would still be the same, since $comp(s) \not\models [\mathcal{E}_{flip}] l$ and $comp(s) \not\models [\mathcal{E}_{\it{move}}] l$.   
%
\end{example}

\section{Extending learning beyond explicit knowledge\label{sec:extending}}
We have presented a learning algorithm for learning explicit knowledge, and illustrated how it works on the light switch domain. How about implicit learning in that domain? Is there more a learner would be able to learn---or deduce---about the domain than what can be directly observed in each state? In this domain, there isn't. The learner of explicit knowledge has already identified the underlying transition system up to isomorphism, so it is only the identification of which propositions are true in each state that are missing. Taking another look at Figure~\ref{figu:lightswitch_compdomain}, it can be seen that this additional insight can not be achieved. Consider for any choice of states $s''_i \in comp(s_i)$, $i= 1,\dots,3$, the domain with states $\{s''_1,\dots,s''_3\}$ and with transition function and observation function induced by the compatibility domain of Figure~\ref{figu:lightswitch_compdomain}. No sequence of observations of action executions can make the agent distinguish this new domain from the real domain (the two domains are bisimilar, a notion to be formally introduced in the next section). Take for instance $s''_0 = s'_0$, $s''_1 = s'_1$, $s''_2 = s_2$ and $s''_3 = s_3$. In the domain induced by these states, the light is initially turned on ($s'_0$), and flipping the switch makes it turn off ($s'_1$). So the function of the switch has been reversed, at least when the agent is in the left room. However, if moving into the other room with the switch up, the light will actually go on again ($s_2$), and if moving with the switch down, the light will turn off ($s_3$). So moving from one room to the other reverses the function of the switch (making you perhaps reconsider which electrician to call for your next electrical wiring job). For each choice of the $s''_i$, we can make a similar description of a domain that would make (some) sense and that would be consistent with any observations that can be made in the domain (bisimilar with the real domain). One could consider learners with an inductive bias that would make them choose certain domains over others. An inductive bias could for instance be the simplicity (size) of the produced set of event models, a kind of basic Ockham's razor principle. However, it can be shown that in the light switch domain, such an inductive bias will still not allow the learner to uniquely settle on the correct domain. To sum up, the light switch domain doesn't allow us to see any difference between learning explicit or implicit knowledge. Let us know consider another domain that does.
\begin{example}\label{example:magic_box} 
\begin{figure}
 \begin{minipage}{0.45\textwidth}
 \[
 \scalebox{0.9}{
 \begin{tikzpicture}[auto,align=center]
  \node[circle,draw,inner sep=-5pt,minimum size=1.1cm,label={below:$comp(s_0)$ \\ $Obs =\top$}] (s0) at (0,0) {
 \begin{tikzpicture}[auto,align=center,inner sep=0pt] 
   \node (s01) {$s_0: \neg p$};
   \node[below of=s01,node distance=5mm] (s02) {$s_1: p$}; 
  \end{tikzpicture}
  };
    \node[right of=s0,node distance=30mm,circle,draw,inner sep=-1pt,minimum size=1.1cm,label={below:$comp(s_1)$ \\ $Obs= p$}] (s3) {
    \begin{tikzpicture}[inner sep=0pt]
       \node (s21) {$s_1: \underline{p}$};
    \end{tikzpicture}
    };
    \path[latex-latex] (s0) edge node {$\it{flip}$} (s3);
 \end{tikzpicture}
 }
 \]
  \caption{The compatibility domain $\mathcal{D}'$ induced by the box domain $\mathcal{D}$ of Example~\ref{example:magic_box}. } \label{figu:magicbox_compdomain}
  \end{minipage} \hfill
  \begin{minipage}{0.45\textwidth}
  \[
 \scalebox{0.9}{
 \begin{tikzpicture}[auto,align=center]
  \node[circle,draw,inner sep=-3pt,minimum size=1.1cm,label={below:$f(comp(s_0))$ \\ $Obs =\top$}] (s0) at (0,0) {
 \begin{tikzpicture}[auto,align=center,inner sep=0pt] 
   \node (s01) {$s_0: \neg p$};
  \end{tikzpicture}
  };
    \node[right of=s0,node distance=30mm,circle,draw,inner sep=-3pt,minimum size=1.1cm,label={below:$f(comp(s_1))$ \\ $Obs= p$}] (s3) {
    \begin{tikzpicture}[inner sep=0pt]
       \node (s21) {$s_1: \underline{p}$};
    \end{tikzpicture}
    };
    \path[latex-latex] (s0) edge node {$\it{flip}$} (s3);
 \end{tikzpicture}
 }
 \]
 \caption{The unique observation determinisation of the compatibility domain $\mathcal{D}'$ shown left.} \label{figu:magicbox_obsdet}
  \end{minipage}
  \end{figure}
Consider the domain $\mathcal{D} = ((S,A,T,s_0),\obs,\obsmod)$ over $P = \{ p \}$ with $S = \{ s_0, s_1 \}$, $s_0 = \emptyset$, $s_1 = \{ p \}$, $A = \{ \it{flip}  \}$, $T = \{(s_0, \it{flip},s_1), (s_1, \it{flip}, s_0)\}$, $Obs(s_0) = \top$, and $Obs(s_1) = p$ (that is, $Obs(s_0) = (\emptyset,\emptyset)$ and $Obs(s_1) = (\{p\},\emptyset)$). So the $\it{flip}$ action flips the truth value of $p$, but the truth value is only observable when true. It sounds perhaps a bit esoteric, since if $p$ is observed when true, why don't we also observe it when false? We can however think of a concrete example of this type. Consider a box that can either be empty ($\neg p$) or full ($p$), and $\it{flip}$ is the action of emptying it if full, and making it full if empty. Suppose further that when the box is empty its walls are completely opaque, but when filled, the pressure on the bottom of the box activates a switch that turns on a light inside the box so that its walls become transparent and it becomes visible that it is full. In this case, when it is empty nothing is observed, $Obs(s_0) = \top$, but when full, it is observed to be full, $Obs(s_1) = p$. 

A learner that is initially unfamiliar with the electronics inside the box and has only been exposed to the state where it is empty and hence opaque, will of course both consider it possible that it is empty and that it is full. So there should be no explicit knowledge about $p$ in that state. However, if the learner has also been exposed to the state where it is full, and has observed $p$ in that state, the learner should be able to conclude that when nothing is observed, $p$ is false (the learner is supposed to be aware that both the transition and observation functions are deterministic). So it should be possible for a learner to come to implicitly know the full dynamics of the domain. 

Let us try to make these things a bit more formal. Consider first the induced compatibility domain $\mathcal{D}' = ((S',A,T',s'_0),\obs,\obsmod')$ of $\mathcal{D}$, shown in Figure~\ref{figu:magicbox_compdomain}. Now note that both states of the compatibility domain contain the state $s_1$. A learner having identified $\mathcal{D}'$ from its interactions with the domain will then conclude that two distinct possible observations are possible in the state $s_1$: the empty observation $\top$ (from the occurrence of the state $s_1$ in $comp(s_0)$) and the observation $p$ (from the occurrence of the state $s_1$ in $comp(s_1)$). Assuming the agent knows the observation function to be deterministic, this is clearly a contradiction. A learner should be able to use the additional knowledge of the observation function being deterministic to refine its representation of the domain. 

We define an \emph{observation determinisation} of a compatibility domain to be any domain $\mathcal{D}''$ defined as follows. The domain of $\mathcal{D}''$ contains for each $x \in S'$ a state $g(x) \subseteq x$ such that the set of states $\{ g(x) \mid x \in S'\}$ form a partition of $S$. The relations and functions on $\mathcal{D}''$ are then inherited from $\mathcal{D}'$ in the canonical way. The only way $\mathcal{D}''$ differs from $\mathcal{D}'$ is that we have removed elements of the states of $\mathcal{D}'$ in order to ensure that each original state $s \in S$ appears in exactly one state of $\mathcal{D}''$. Since each pair of states of $\mathcal{D}'$, and hence $\mathcal{D}''$, have distinct observations, this ensures that $\mathcal{D}''$ only specifies a single observation for each original state $s \in S$. Considering the domain above as an example, there only exists a single observation determinisation given by $g(comp(s_0)) = \{ s_0 \}$ and $g(comp(s_1)) = \{ s_1 \}$. The corresponding domain is shown in Figure~\ref{figu:magicbox_obsdet}. We can now of course also define a notion of behavioural correctness with respect to a observation determinisation, and use that to define a notion of implicit knowledge. We will not introduce the technical details here, but just mention that it would be possible to learn event models with two relations $R_e$ and $R_i$, corresponding to two modalities, $K_e$ for explicit knowledge and $K_i$ for implicit knowledge. We could require that explicit knowledge is behaviourally correct with respect to the compatibility domain, and implicit knowledge with respect to its observation determinisation. In that case, a learner would be able to express facts such as:
\[
  comp(s_0) \models K_i \neg p \wedge \neg K_e \neg p
\]
expressing that in the initial state it is implicitly known that the box is empty, but it is not explicitly known. We previously gave examples of the usefulness of being able to express both kinds of knowledge, e.g.\ in multi-agent settings.

The domain considered here is of course very simple. It only has a single observation determinisation, and that determinisation identifies the full underlying transition system. In general, domains can have many observation determinisations, and then we can not expect there to be a single domain to define behavioural correctness with respect to. Behavioural correctness would then have to be defined in terms of behavioural equivalence with respect to a single or all observation determinisations. We leave these consideration for the next section focusing on learning algorithms for implicit knowledge.
\end{example}
The example above considered a notion of refinement of compatibility domains based on knowing that the observation function is deterministic. We will now give an example of a possible way to refine compatibility domains based on knowing that the transition function is deterministic. 
\begin{example}\label{example:knock-knock}
\begin{figure}
 \begin{minipage}{0.45\textwidth}
 \[
 \scalebox{0.9}{
 \begin{tikzpicture}[auto,align=center]
  \node[circle,draw,inner sep=-7pt,minimum size=1.6cm,label={below:}] (s0) at (0,0) {
 \begin{tikzpicture}[auto,align=center,inner sep=0pt] 
   \node (s01) {$s_0: p \underline{q}$};
  \end{tikzpicture}
  };
    \node[right of=s0,node distance=40mm,circle,draw,inner sep=-7pt,minimum size=1.6cm,label={below:}] (s3) {
    \begin{tikzpicture}[inner sep=0pt]
       \node (s21) {$s_1: \neg p \underline{q}$};
    \end{tikzpicture}
    };
    \path[-latex] (s0) edge node {$a$} (s3);
   \node[below right of=s0,node distance=29mm,circle,draw,inner sep=-7pt,minimum size=1.6cm,label={below:}] (s2) {
    \begin{tikzpicture}[inner sep=0pt]
       \node (s21) {$s_2: \neg p \underline{\neg q}$};
    \end{tikzpicture}
    };
    \path[-latex] (s3) edge node {$a$} (s2);
    \path[-latex] (s2) edge node {$a$} (s0);    
 \end{tikzpicture}
 }
 \]
  \caption{The door knocking domain $\mathcal{D}$ of Example~\ref{example:knock-knock}. } \label{figu:knockknock_domain}
  \end{minipage} \hfill
  \begin{minipage}{0.45\textwidth}
  \[
 \scalebox{0.9}{
\begin{tikzpicture}[auto,align=center]
  \node[circle,draw,inner sep=-10pt,minimum size=1.7cm,label={below:$comp(s_0)$ \\ $Obs = q$}] (s0) at (0,0) {
 \begin{tikzpicture}[auto,align=center,inner sep=0pt] 
   \node (s01) {$s_0: p \underline{q}$};
   \node[below of=s01,node distance=5mm] (s02) {$s_1: \neg p \underline{q}$}; 
  \end{tikzpicture}
  };
  \draw[-latex] (s0) edge [loop above,->,>=latex] node {$a$} (s0);
    \node[right of=s0,node distance=40mm,circle,draw,inner sep=-10pt,minimum size=1.7cm,label={below:$comp(s_2)$ \\ $Obs= \neg q$}] (s3) {
    \begin{tikzpicture}[inner sep=0pt]
       \node (s21) {$s_2: \neg p \underline{\neg q}$};
       \node[below of=s21,node distance=5mm] (s22) {$p \underline{\neg q}$}; 
    \end{tikzpicture}
    };
    \path[latex-latex] (s0) edge node {$a$} (s3);
 \end{tikzpicture}
 }
 \]
 \caption{The compatibility domain $\mathcal{D}'$ induced by the door knocking domain $\mathcal{D}$ shown left. Note that $comp(s_2)$ contains the state $\{ p \}$ (denoted as $p \underline{\neg q}$), which is not contained in the state space of the domain $\mathcal{D}$.} \label{figu:knockknock_compdomain}
  \end{minipage}
  \end{figure}
Consider again the earlier sketched domain over $P = \{p,q\}$ with a single action $a$ that produces the following sequence of observations: $q, q, \neg q, q, q, \neg q, \dots$. For instance, it could be a domain where $a$ is an action of knocking on a door that is initially closed ($q$), and only when knocking twice will it open ($\neg q$). Knocking once when open then closes it again. Formally, it could be represented as a domain  $\mathcal{D} = ((S,A,T,s_0),\obsmod,\obs)$ with $S = \{ s_0, s_1, s_2 \}$,
$s_0 = \{p, q \}$, $s_1 = \{ q \}$, $s_2 = \emptyset$, $A = \{a\}$, $T = \{ (s_0,a,s_1), (s_1, a, s_2), (s_2, a, s_0) \}$ and $Obs(s_0) = Obs(s_1) = q$ and $Obs(s_2) =\neg q$. Note that, as mentioned earlier, the truth value of $q$ is always observed, and the truth value of $p$ is never observed. The proposition $p$ being true encodes that the door hasn't been knocked at since it was last closed. The domain $\mathcal{D}$ is presented in Figure~\ref{figu:knockknock_domain} and its compatibility domain in Figure~\ref{figu:knockknock_compdomain}. Note that the compatibility domain is non-deterministic, since the state $comp(s_0)$ has two outgoing $a$-edges. If the learner knows the underlying domain to be deterministic, it should be able to refine $\mathcal{D}'$ into a deterministic domain. In this case, we cannot refine the compatibility domain by simply modifying its existing states. There are simply too few (compatibility) states to make the transition function deterministic. 

We could probably alternatively define ways to ``unfold'' compatibility domains. Note that in this particular example, the compatibility domain is obtained simply by identifying the two upper states of Figure~\ref{figu:knockknock_domain} (but in general the compatibility domain is not defined as a simple quotient in this way, as the compatibility function doesn't always induce an equivalence relation on states as the previous example showed). As for observation determinisation, we would in general then get many distinct ways of unfolding a compatibility domain. In this example, one unfolding would produce the real domain, and another would produce the one with $p$ swapped by $\neg p$ everywhere.
We are not going to pursue the technical details of defining such unfoldings in this paper. To define them in detail, we would need to take into account that whether an unfolding is consistent with the observations received is not only a matter of which individual transitions have been observed, but also the order of transitions. Unfolding the domain of Figure~\ref{figu:knockknock_compdomain} can clearly be done in many different ways, also in ways in which the frequency of observing $\neg q$ is different, e.g.\ only observing $\neg q$ every fourth time (though that would require a bigger language $P$). Observing streams of action executions is here needed to determine which domain unfoldings are consistent with the actual underlying domain. Furthermore, to produce all deterministic domains consistent with a given compatibility domain, we would need to consider both observation determinisations (as in the previous example) and unfoldings, potentially even interleaved. Defining algorithms for producing these domains, and then proving them to have the expected properties, is probably possible, but also non-trivial. Instead, we are in the next section going to define a new algorithm for learning implicit knowledge that in a more direct way builds all the domains consistent with the observed transitions, also taking the order of action execution into account. 
\end{example}

\section{Learning implicit domain knowledge\label{sec:implicit}}

At the beginning of Section \ref{sec:explicit}, we informally described explicit knowledge as what is known because it is directly observed in the current state, and implicit knowledge as whatever might additionally be inferred from the history of earlier  actions and from the general experience with the domain. To formalise the notion of explicit knowledge, we introduced the notion of a \emph{compatibility domain}: a domain that captures what is explicitly knowable about the real domain, and whose states are sets of states over the alphabet of the domain. A learner was then deemed behaviourally correct with respect to explicit knowledge if the learner outputs a set of event models whose induced domain is isomorphic to the compatibility domain. 

In this section, we proceed in a similar way with respect to implicit domain knowledge. To capture it, we introduce the notion of a \emph{behavioural equivalence domain}. A learner is deemed \emph{behaviourally correct with respect to implicit knowledge} if it outputs a domain that is isomorphic to the behavioural equivalence domain, and not, as before, to the compatibility domain.


\subsection{Implicit knowledge and behavioural equivalence}

To formalise the notion that implicit knowledge is whatever can be inferred from a history of earlier actions and general experience with the domain, we resort to the concept of  \emph{behavioural equivalence}.  Many different notions of \emph{behavioural equivalence} have been proposed in the literature on labelled transition systems~\cite[Ch.~2]{gorrieri2015introduction}.  Behavioural equivalence relations seek to establish in which cases two transition systems offer similar interaction capabilities. The intuition is that two systems should be equivalent if they cannot be distinguished by interacting with them. If the initial states of two systems are behaviourally equivalent, then they cannot be distinguished by any experience gathered from these initial states, through any sequence of actions. 

In this section, we introduce notions of behavioural equivalence for domains, based on what the agent can observe about the underlying transition system. Two domains will be deemed equivalent if they cannot be distinguished, \textit{from observations}, by interacting with them. In other words, we will formalise implicit knowledge as knowledge up to behavioural equivalence. This is in fact as much knowledge as a learner can possibly acquire. Learning proceeds by making ``experiments'' with, or ``testing'' the domain, i.e.~by taking certain sequences of actions and observing the results. But, in two behaviourally equivalent domains, any such experiment yields the same observations. As a result, the agent can only come to know with certainty the information that holds in every domain that's equivalent to the real one.

To define the behavioural equivalence relations used in this section, we need to fix some notions. 
 Let $\mathcal{D} = ((S,A,T,s_0),\obs, \obsmod)$ be a deterministic domain over $P$. Let $A^*$ denote the set of all finite sequences of elements of $A$. An element of $A^*$ is called an action sequence or \emph{action trace}~\cite{gorrieri2015introduction}. Each such trace defines an interaction with the system, e.g.,~the trace $(move,move,flip,move)$ is an interaction in the light switch domain of Example \ref{example:lightswitch_domain}. 
We denote by $\textnormal{Tr}(\mathcal{D})$ the set of traces in $\mathcal{D}$.

The \emph{execution trace of $(a_i)_{0\leq i \leq n}$ from $s$} is the sequence
$(t_0,a_0,t_1,a_1,\dots,t_n,a_n,t_{n+1})$
where $t_0 = s$ and $(t_i,a_i,t_{i+1}) \in T$ for all $0\leq i \leq n$. That is, an execution trace is an alternating sequence of states and actions ending with a state. We denote 

 the set of executions traces from $s$ by $\textnormal{ExTr}(\mathcal{D},s) \coloneqq \{\varepsilon \mid \varepsilon \text{ is the execution trace from $s$ of some $\alpha \in \textnormal{Tr}(\mathcal{D})$} \}$. 
The \emph{observation trace for} an execution trace $$(s_0,a_0,s_1,\dots,a_n,s_{n+1})$$ 
is the sequence $$(Obs(s_0),a_0,Obs(s_{1}),\dots,a_n,Obs(s_{n+1})).$$
 
We denote the set of observation traces from $s$ by \[\textnormal{ObsTr}(\mathcal{D},s) \coloneqq \{ \tau \mid \tau \text{ is the } \text{observation trace for } \varepsilon, \text{ for some } \varepsilon\in\textnormal{ExTr}(\mathcal{D},s) \}\] 
%
The first behavioural equivalence we discuss is \emph{trace equivalence}, based on equivalence of traces up to what can be observed.

\begin{definition} Let $\mathcal{D} = ((S,A,T,s_0),\obs, \obsmod)$ and $\mathcal{D}' = ((S',A',T',s'_0),\obs', \obsmod')$ be two deterministic domains over $P$. Two states $s\in\states$ and $s'\in\states'$ are called \emph{(observationally) trace equivalent} if $\textnormal{ObsTr}(\mathcal{D},s)=\textnormal{ObsTr}(\mathcal{D'},s')$. $\mathcal{D}$ and $\mathcal{D}'$ are called \emph{(observationally) trace equivalent} if their initial states $s_0$ and $s'_0$ are observationally trace equivalent.
\end{definition}
For transition systems, the classical alternative to trace equivalence is bisimilarity. We introduce a version of bisimulation for domains which relates two states $s$ and $t$ when $s$ and $t$ are observationally indistinguishable, and such that if the same action is executed in both states,  the resulting states are again observationally indistinguishable.

\begin{definition} Let $\mathcal{D}=((\states,\actions,\transmod,s_0),\obs,\obsmod)$ and $\mathcal{D}'=((\states',\actions,\transmod',s'_0), \obs, \obsmod')$ be two deterministic domains over $P$. A relation $Z\subseteq S\times S'$ is called a \emph{bisimulation} 
between $S$ and $S'$ if for every $s\in S$ and $s'\in S'$, the following conditions hold:
\begin{description}
    \item[Observational indistinguishability:] if $sZs'$, then  $Obs(s)=Obs'(s)$. 
    \item[Forth:] if $sZs'$ and $(s,a,t)\in T$, then there exists a $t'\in S'$ s.t. $(s',a,t')\in T'$ and $tZt'$.
    \item[Back:] if $sZs'$ and $(s',a,t')\in T'$, then there exists a $t\in S$ s.t. $(s,a,t)\in T$ and $tZt'$.
\end{description}
Two states $s\in S$ and $s'\in S$ are called \emph{(observationally) bisimilar}, denoted $s\bisim s'$, if there is a bisimulation $Z$ such that $sZs'$. The domains $\mathcal{D}$ and $\mathcal{D}'$ are called \emph{(observationally) bisimilar}, denoted $\mathcal{D}\bisim \mathcal{D}'$, if their starting states are observationally bisimilar, i.e.~$s_0\bisim s'_0$ . 
\end{definition}

For classical, deterministic transition systems, trace equivalence and bisimilarity coincide~\cite{gorrieri2015introduction}, trace equivalence simply means that the same action traces are possible (the same action sequences are applicable). In our case, for deterministic domains, trace equivalence and bisimilarity are based on a different notion of equivalence under observability. However, we still get that (observational) trace equivalence and (observational) bisimilarity coincide.

\begin{lemma} \label{lemma:bisim_trace} Two deterministic domains $\mathcal{D}$ and $\mathcal{D}'$ over $P$ are trace equivalent iff they are bisimilar.
\end{lemma}
\begin{proof}
Let $\mathcal{D}=((\states,\actions,\transmod,s_0),\obs,\obsmod)$ and $\mathcal{D}'=((\states',\actions,\transmod',s'_0), \obs, \obsmod')$ be two deterministic domains over $P$.

$(\Rightarrow)$ Suppose that $\mathcal{D}$ and $\mathcal{D}$ are trace equivalent. Then $\textnormal{ObsTr}(\mathcal{D},s_0)=\textnormal{ObsTr}(\mathcal{D'},s'_0)$. Define a relation $Z$ by $sZ s'$ iff $\textnormal{ObsTr}(\mathcal{D},s)=\textnormal{ObsTr}(\mathcal{D'},s')$. Thus, $Z$ relates $s_0$ and $s'_0$ by definition. To show that $Z$ is a bisimulation, we show conditions (i)-(iii) for bisimulations. Suppose that $s Z s'$. For condition (i), take any observation trace $(o,\dots) \in ObsTr(D,s)$. Then $Obs(s) = o$. Since $sZs'$, we have $ObsTr(D,s) = ObsTr(D',s')$ and hence $(o,\dots) \in ObsTr(D',s')$. This implies $Obs'(s') = o$, and hence $Obs'(s') = o = Obs(s)$, as required.
For condition (ii), suppose that $s Z s'$ and $(s,a,t)\in T$. Since $(s,a,t)\in T$, we get $(s,a,t,\dots)\in \textnormal{ExTr}(\mathcal{D},s)$. Hence, $(Obs(s),a,Obs(t),\dots)\in \textnormal{ObsTr}(\mathcal{D},s)$, and since $\textnormal{ObsTr}(\mathcal{D},s)=\textnormal{ObsTr}(\mathcal{D'},s')$, we get $(Obs'(s'),a,Obs'(t'),\dots)\in\textnormal{ObsTr}(\mathcal{D}',s')$, for some $t'\in S$. Since $Obs(s)=Obs'(s')$, $Obs(t)=Obs'(t')$ and $\textnormal{ObsTr}(\mathcal{D},s)=\textnormal{ObsTr}(\mathcal{D'},s')$, we get
\begin{align*}
    \textnormal{ObsTr}(\mathcal{D},t) = & \ \{ (Obs(t),\dots) \mid (Obs(s),a,Obs(t),\dots) \in \textnormal{ObsTr}(\mathcal{D},s)\} \\
    = & \ \{ (Obs'(t'),\dots) \mid (Obs'(s'),a,Obs'(t'),\dots) \in \textnormal{ObsTr}(\mathcal{D}',s')\} \\
    = & \ \textnormal{ObsTr}(\mathcal{D}',t').
\end{align*}
Condition (iii) is symmetric.

($\Leftarrow$) Suppose that $\mathcal{D} \bisim \mathcal{D}'$, i.e.~there is a bisimulation $Z$ such that  $s_0 Z s'_0$. To show $\textnormal{ObsTr}(\mathcal{D},s_0)=\textnormal{ObsTr}(\mathcal{D'},s'_0)$, we prove inclusion both ways. Let $(Obs(s_0),a_0,Obs(s_{1}),\dots,a_n,Obs(s_{n+1})) \in \textnormal{ObsTr}(\mathcal{D},s_0)$. Then 
 $(s_0,a_0,s_{1},\dots,a_n,s_{n+1})\in \textnormal{ExTr}(\mathcal{D},s_0)$. By the forth condition, there are $s'_1,\dots, s'_{n+1}\in\states'$ such that $(s'_i,a_i,s'_{i+1})\in T'$, $s_i Z s'_i$, for all $0\leq i \leq n$. Since $s_i Z s'_i$, $Obs(s_i) = Obs'(s'_i)$ for all $0\leq i \leq n$. Thus, $(Obs(s_0),a_0,Obs(s_{1}),\dots,a_n,Obs(s_{n+1}))\in \textnormal{ObsTr}(\mathcal{D}',s'_0)$. The proof for $\textnormal{ObsTr}(\mathcal{D'},s'_0)\subseteq\textnormal{ObsTr}(\mathcal{D},s_0)$ is analogous to the one just given, except for using the back condition instead of the forth one.
\end{proof}

\begin{example}\label{example:four-bisimilar}
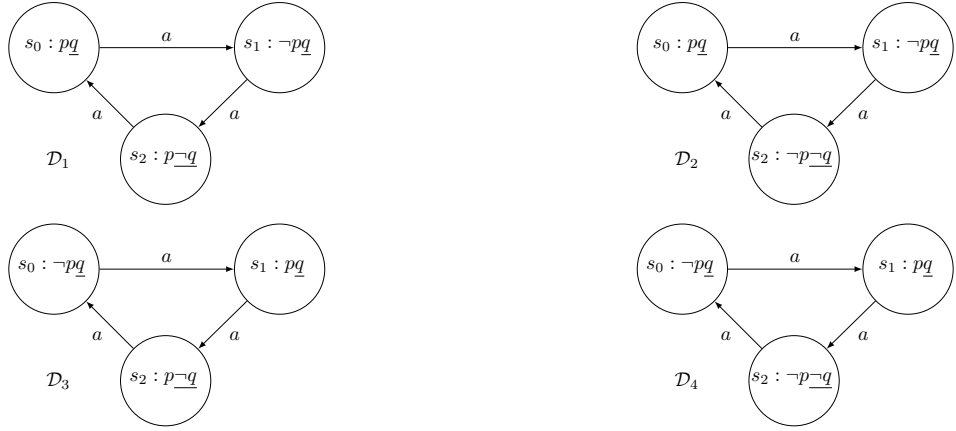
\begin{figure}
\begin{minipage}{0.45\textwidth}
 \[
 \scalebox{0.75}{
 \begin{tikzpicture}[auto,align=center]
  \node[circle,draw,inner sep=-10pt,minimum size=1.6cm,label={below:}] (s0) at (0,0) {
 \begin{tikzpicture}[auto,align=center,inner sep=0pt] 
   \node (s01) {$s_0: p \underline{q}$};
  \end{tikzpicture}
  };
    \node[right of=s0,node distance=40mm,circle,draw,inner sep=-10pt,minimum size=1.6cm,label={below:}] (s3) {
    \begin{tikzpicture}[inner sep=0pt]
       \node (s21) {$s_1: \neg p \underline{q}$};
    \end{tikzpicture}
    };
    \path[-latex] (s0) edge node {$a$} (s3);
   \node[below right of=s0,node distance=28mm,circle,draw,inner sep=-10pt,minimum size=1.6cm,label={below:}] (s2) {
    \begin{tikzpicture}[inner sep=0pt]
       \node (s21) {$s_2: p \underline{\neg q}$};
    \end{tikzpicture}
    };
    \path[-latex] (s3) edge node {$a$} (s2);
    \path[-latex] (s2) edge node {$a$} (s0);    
    \node[left of=s2,xshift=-9mm] {$\mathcal{D}_1$};
 \end{tikzpicture}
 } \]
  \end{minipage} \hfill 
\begin{minipage}{0.45\textwidth}
\[
\scalebox{0.75}{
\begin{tikzpicture}[auto,align=center]
\node[circle,draw,inner sep=-10pt,minimum size=1.6cm,label={below:}] (s0) at (0,0) {
\begin{tikzpicture}[auto,align=center,inner sep=0pt] 
\node (s01) {$s_0: p \underline{q}$};
\end{tikzpicture}
};
\node[right of=s0,node distance=40mm,circle,draw,inner sep=-10pt,minimum size=1.6cm,label={below:}] (s3) {
\begin{tikzpicture}[inner sep=0pt]
   \node (s21) {$s_1:\neg p \underline{q}$};
\end{tikzpicture}
};
\path[-latex] (s0) edge node {$a$} (s3);
\node[below right of=s0,node distance=28mm,circle,draw,inner sep=-10pt,minimum size=1.6cm,label={below:}] (s2) {
\begin{tikzpicture}[inner sep=0pt]
   \node (s21) {$s_2:\neg p \underline{\neg q}$};
\end{tikzpicture}
};
\path[-latex] (s3) edge node {$a$} (s2);
\path[-latex] (s2) edge node {$a$} (s0);  
\node[left of=s2,xshift=-9mm] {$\mathcal{D}_2$};
\end{tikzpicture}
}
\] 
\end{minipage} 
\begin{minipage}{0.45\textwidth}
 \[
 \scalebox{0.75}{
 \begin{tikzpicture}[auto,align=center]
  \node[circle,draw,inner sep=-10pt,minimum size=1.6cm,label={below:}] (s0) at (0,0) {
 \begin{tikzpicture}[auto,align=center,inner sep=0pt] 
   \node (s01) {$s_0: \neg p \underline{q}$};
  \end{tikzpicture}
  };
    \node[right of=s0,node distance=40mm,circle,draw,inner sep=-10pt,minimum size=1.6cm,label={below:}] (s3) {
    \begin{tikzpicture}[inner sep=0pt]
       \node (s21) {$s_1:p \underline{q}$};
    \end{tikzpicture}
    };
    \path[-latex] (s0) edge node {$a$} (s3);
   \node[below right of=s0,node distance=28mm,circle,draw,inner sep=-10pt,minimum size=1.6cm,label={below:}] (s2) {
    \begin{tikzpicture}[inner sep=0pt]
       \node (s21) {$s_2: p \underline{\neg q}$};
    \end{tikzpicture}
    };
    \path[-latex] (s3) edge node {$a$} (s2);
    \path[-latex] (s2) edge node {$a$} (s0);    
    \node[left of=s2,xshift=-9mm] {$\mathcal{D}_3$};
 \end{tikzpicture}
 }
 \] 
  \end{minipage} \hfill
\begin{minipage}{0.45\textwidth}
 \[
 \scalebox{0.75}{
 \begin{tikzpicture}[auto,align=center]
  \node[circle,draw,inner sep=-10pt,minimum size=1.6cm,label={below:}] (s0) at (0,0) {
 \begin{tikzpicture}[auto,align=center,inner sep=0pt] 
   \node (s01) {$s_0: \neg p \underline{q}$};
  \end{tikzpicture}
  };
    \node[right of=s0,node distance=40mm,circle,draw,inner sep=-10pt,minimum size=1.6cm,label={below:}] (s3) {
    \begin{tikzpicture}[inner sep=0pt]
       \node (s21) {$s_1: p \underline{q}$};
    \end{tikzpicture}
    };
    \path[-latex] (s0) edge node {$a$} (s3);
   \node[below right of=s0,node distance=28mm,circle,draw,inner sep=-10pt,minimum size=1.6cm,label={below:}] (s2) {
    \begin{tikzpicture}[inner sep=0pt]
       \node (s21) {$s_2: \neg p \underline{\neg q}$};
    \end{tikzpicture}
    };
    \path[-latex] (s3) edge node {$a$} (s2);
    \path[-latex] (s2) edge node {$a$} (s0);  
    \node[left of=s2,xshift=-9mm] {$\mathcal{D}_4$};
 \end{tikzpicture}
 }
 \] 
  \end{minipage} 
\caption{The domains $\mathcal{D}_1,\dots,\mathcal{D}_4$ over $P=\{p,q\}$ that are bisimilar to the door knocking domain of Example \ref{example:knock-knock}. The real domain is denoted here as $\mathcal{D}_2$.}
\label{fig:all_bisim}
\end{figure}
There are four domains that are bisimilar, i.e.~trace equivalent, to the door knocking domain of Example \ref{example:knock-knock}. Figure \ref{fig:all_bisim} shows them, where the real domain is $\mathcal{D}_2$. The two top domains, $\mathcal{D}_1$ and $\mathcal{D}_2$, differ from the two bottom ones, $\mathcal{D}_3$ and $\mathcal{D}_4$, in the initial state. The left domains differ from the right ones in the bottom state. To see that these are in fact the only bisimilar or trace equivalent domains, recall first that, by definition, domains over a set of propositional atoms $P$ have a state space $S\subseteq 2^P$. This means that any domain over $P$ that is trace equivalent to the original one will have at most four states. (In general, since the set of actions in a domain $\mathcal{D}$ over $P$ is finite, and the set of states is finite as well, the set of domains bisimilar to $\mathcal{D}$ is
finite). 

Note that a domain with a single state $s^*_0$ cannot be trace equivalent to the original domain $\mathcal{D}_2$, since it must be able to produce the observation trace $(q,a,q,a,\neg q)$, which would require that both $q$ and $\neg q$ are observable in $s^*_0$. Now consider any domain with two states, $s^*_0$ and $s^*_1$, and suppose it was bisimilar to the real one. By the observational indistinguishability condition of bisimulations, and the fact that in the initial state of the door knocking $q$ is observed as true, $s^*_0$ must also be a state in which $q$ is observed as true. Since the domain must be able to produce the observation trace $(q,a,q,a,\neg q)$, there must be a state in which $\neg q$ is observed. This must be $s^*_1$. In order to produce the observation trace $(q,a,q)$, this domain would need to have $(s^*_0,a,s^*_0)$ in its transition function. But then, in order to produce the observation trace $(q,a,q,a,\neg q)$, the domain would need to have $(s^*_0,a,s^*_1)$ in its transition function. The transition function would then be non-deterministic. But bisimilarity is only defined for deterministic domains over the same set of propositional variables $P$. 

Next, consider domains with three states, $s^*_0$, $s^*_1$ and $s^*_2$. It is straightforward to check that the domains in Figure \ref{fig:all_bisim} are bisimilar to the real domain; the relation linking $s_i$ to $s^*_i$ is a bisimulation. Now take any three-state domain not listed in Figure \ref{fig:all_bisim} and suppose it is bisimilar to the real domain. In order to produce the observation trace $(q,a,q)$, the $a$-successor of $s^*_0$ must be a $q$-state. Suppose the $a$-successor of $s^*_0$ is $s^*_0$ itself. Then, as with domains
with two states, in order to produce the observation trace $(q,a,q,a,\neg q)$, the domain would need a non-deterministic transition function. Hence, the $a$-successor of $s^*_0$ must be a $q$-state different from $s^*_0$. W.l.o.g., let this state be $s^*_1$.
In order to produce the observation trace $(q,a,q,a,\neg q)$, $s^*_2$ must be the $a$-successor of $s^*_1$ and it must be a $\neg q$-state. Moreover, in order to produce the observation trace $(q,a,q,a,\neg q, a, q)$, the $a$-successor of  $s^*_2$ must be a $q$-state. If the successor was  $s^*_0$, the domain would be listed in Figure \ref{fig:all_bisim}. So the $a$-successor of $s^*_2$ must be a $s^*_1$. But note that then, in this domain, the following would be a possible observation trace:
$(q,a,q,a,\neg q, a, q,a,\neg q)$. This trace is however impossible in the door knocking domain; $\neg q$ observations occur always after \textit{two} $q$ observations. Such a domain would break the repeated observation pattern $q, q, \neg q, q, q, \neg q, \dots$\footnote{Note that an observation trace  re-visiting states is needed to notice that the domain under consideration is not bisimilar to the real one. Later on, in this section, we will bound the length of observation traces required to tell with certainty whether a domain is bisimilar or not to the one being explored.}

We are thus left with domains with four states: $s^*_0$, $s^*_1$, $s^*_2$ and $s^*_3$. Reasoning as before, we can show that the $a$-successor of $s^*_0$ must be a $q$-state different from $s^*_0$. W.l.o.g., let's say that this state is $s^*_1$. In order to produce the observation trace $(q,a,q,a,\neg q)$, the $a$-successor of $s^*_1$ must be a $\neg q$-state. W.l.o.g., let this state be $s^*_2$. Again, in order to produce the observation trace $(q,a,q,a,\neg q, a, q)$, the $a$-successor of $s^*_2$ must be a $q$-state. If this successor is $s^*_1$, we get a non-bisimilar domain. If this successor is $s^*_0$, then $s^*_3$ is not reachable from $s^*_0$, which cannot be, by definition of domains. And if this successor is $s^*_4$, then the domain produces the observation trace $(q,a,q,a,\neg q, a, \neg q)$, which is not a possible observation trace for the real domain.

\end{example}

\subsection{Behavioural equivalence domain}

We now have all prerequisite notions to define the behavioural equivalence domain. Before defining the domain formally, we provide some intuition about it. Consider an agent situated in the initial state $s_0$ of a domain $\mathcal{D}$. Since the agent only has access to observations of states, it cannot distinguish $s_0$ from the initial state $s'_0$ of any other bisimilar domain $\mathcal{D}'$. The agent therefore considers $s'_0$ as a possible world. That is, the agent's initial uncertainty can be represented by the set $W_0$ of initial states from all domains that are bisimilar to $\mathcal{D}$. Now, the agent is unable to tell, just from its observations, whether it is interacting with the real domain, or a bisimilar alternative. Thus, after executing any action $a$ in $s_0$, the agent won't be able to distinguish the $a$-successor $s_1$ of $s_0$, from the $a$-successor $s'_1$ of the bisimilar state $s'_0$. That is, after $a$ is executed, the agent's new state of uncertainty can be represented by the set of $a$-successors of the states in $W_0$. The definition of the behavioural equivalence domain follows this intuition.

\begin{definition} \label{definition:behav_equiv_domain} Let $\mathbb{D}=\{((S^i,A,T^i,s^i_0), \obs,\obsmod^i) \mid 1\leq i\leq n \}$ be a finite set of domains, all bisimilar. The \emph{synchronous composition} of (the domains in) $\mathbb{D}$ is the domain $\mathcal{D}'=((S',A,T',s'_0), \obs,\obsmod')$ given by 
\begin{itemize}
  \item $(s^i_0)_{1\leq i\leq n}\in S'$ and for all $a\in A$: if $(s^i)_{1\leq i\leq n}\in S'$ then $T^i(s^i,a)_{1\leq i\leq n}
  \in S'$.
  \item $T'((s^i)_{1\leq i\leq n},a) = (T^i(s^i,a))_{1\leq i\leq n}$.
  \item $s'_0 = (s^i_0)_{1\leq i\leq n}\in S'$.
  \item $\obsmod'((s^i)_{1\leq i\leq n}) = Obs^1(s^1)$. 
\end{itemize}
When $\mathbb{D}$ is the set of all domains bisimilar to a given domain $\mathcal{D}$, we call $\mathcal{D}'$ the \emph{behavioural equivalence domain} induced by $\mathcal{D}$.
\end{definition}
Our notion of synchronous composition of bisimilar domains is the natural counterpart of the notion of synchronous composition of finite automata~\cite{cassandras2009introduction}: To get from the synchronous compositions of automata to domains, we just need to take the observation functions into account and restrict to the states reachable by synchronous traces from the initial state. We refer to the states of the behavioural equivalence domain as \emph{global states} and to its transitions as \emph{global transitions}. The state space of the behavioural equivalence domain is built recursively. The tuple $(s^i_0)_{1\leq i\leq n}$, which has the initial states of all bisimilar domains as components, is the initial global state. Then for each global state $(s^i)_{1\leq i\leq n}\in S'$ and each action $a\in A$, we add to the set of global states the tuple $T^i(s^i,a)_{1\leq i\leq n}$, which has as $i$-th component the $a$-successor of $s^i$. The global transition function maps the global state $(s^i)_{1\leq i\leq n}\in S'$ and the action $a$ to the global state that has as $i$-th component the $a$-successor of $s^i$.  
The observation function $Obs'$ assigns to the global state $(s^i)_{1\leq i\leq n}\in S'$ the same as the observation that each $Obs^i$ assigns to its $i$-th component.\footnote{The observation function is well-defined, as the components of a global state are all bisimilar (else, the forth condition would be violated) and thus all receive the same observation.} Note that the behavioural equivalence may have two non-identical global states $(s^i)_{1\leq i\leq n}\in S'$ and $(t^i)_{1\leq i\leq n}\in S'$ with the same underlying sets, i.e.~$\{s^i \mid 1\leq i\leq n\} = \{t^i \mid 1\leq i\leq n\}$.

We will now use the behavioural equivalence domain to define our notion of implicit knowledge. The definition is completely symmetric to the definition of explicit knowledge, except based on the behavioural equivalence domain instead of the compatibility domain. In a global state of the behavioural equivalence domain, we consider the implicit knowledge of the agent to be what is true in \emph{all} its component states. To make this notion precise in an epistemic setting, we identify each global state $s=(s^i)_{1\leq i\leq n}$ with an \emph{induced} epistemic model $\mathcal{M}_s = (W,R,V)$ with $W = \{s^i  \mid 1\leq i\leq n \}$, $R = W \times W$ and $V(w) = w$. By this identification, we can evaluate (static) epistemic formulas in global states, e.g.\ we then have $(s^1,\dots,s^n) \models K \phi$ iff $\phi$ is true in every world $s^1,\dots,s^n$. The modal operator $K$ is here interpreted as the operator for implicit knowledge. In the following, we will generally identify global states $s$ with their induced epistemic models $\mathcal{M}_s$ without notice. It will be clear from the context whether a global state should be considered as just that or as the induced epistemic model.  

\subsection{Behavioural correctness and learnability}
A domain $\mathcal{D}$ that is isomorphic to the behavioural equivalence domain represents all that is implicitly knowable about the domain. A learner that identifies such a domain will be called \emph{behaviourally correct with respect to implicit knowledge}.
\begin{definition}
Let $\mathcal{D} = ((S,A,T,s_0),\obs,\obsmod)$ and $\mathcal{D}' = ((S',A,T',s'_0),\obs',\obsmod')$ be domains. We say that $\mathcal{D}'$ is \emph{behaviourally correct with respect to implicit knowledge} about $\mathcal{D}$ if $\mathcal{D}'$ is isomorphic to the behavioural equivalence domain induced by $\mathcal{D}$. A collection of event models $(\mathcal{E}_a)_{a \in A}$ is \emph{behaviourally correct with respect to implicit knowledge} about $\mathcal{D}$ if the domain induced by $(\mathcal{E}_a)_{a \in A}$ and the initial state of the behavioural equivalence domain of $\mathcal{D}$ is behaviourally correct with respect to implicit knowledge about $\mathcal{D}$.
\end{definition}

\begin{example} 
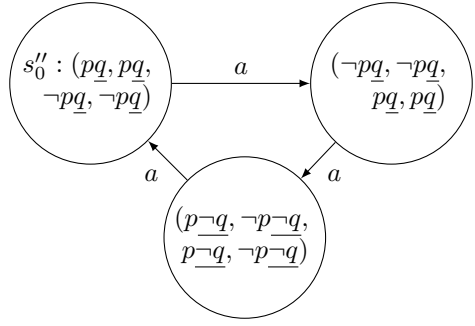
\begin{figure}
\[
\begin{tikzpicture}[auto,align=center]
  \node[circle,draw,inner sep=-7.5pt,minimum size=2.0cm,label={below:}] (s0) at (0,0) {
 \begin{tikzpicture}[auto,align=center,inner sep=0pt] 
   \node[align=right] (s01) {$s''_0: (p \underline{q},p \underline{q},$ \\ $\neg p \underline{q},\neg p \underline{q})$}; 
  \end{tikzpicture}
  };
    \node[right of=s0,node distance=40mm,circle,draw,inner sep=-7.5pt,minimum size=2.0cm,label={below:}] (s3) {
    \begin{tikzpicture}[inner sep=0pt]
       \node[align=right] (s21) {$(\neg p\underline{q},\neg p\underline{q},$ \\ $p\underline{q}, p\underline{q})$}; 
    \end{tikzpicture}
    };
    \path[-latex] (s0) edge node {$a$} (s3);
   \node[below right of=s0,node distance=29mm,circle,draw,inner sep=-7.5pt,minimum size=2.0cm,label={below:}] (s2) {
    \begin{tikzpicture}[inner sep=0pt]
       \node[align=right] (s31) {$(p\underline{\neg q}, \neg p\underline{\neg q},$ \\ $p\underline{\neg q}, \neg p\underline{\neg q})$}; 
    \end{tikzpicture}
    };
    \path[-latex] (s3) edge node {$a$} (s2);
    \path[-latex] (s2) edge node {$a$} (s0);    
 \end{tikzpicture}
\]
\caption{The behavioural equivalence domain $\mathcal{D}''$ induced by the door knocking domain $\mathcal{D}$ of Figure~\ref{figu:knockknock_domain}.}
\label{fig:beh_equiv_real}
\end{figure}
Figure~\ref{fig:beh_equiv_real} depicts the behavioural equivalence domain $\mathcal{D}''$ induced by the door knocking domain $\mathcal{D}$ from Example~\ref{example:knock-knock}. Example~\ref{example:four-bisimilar} proves that there are 4 domains bisimilar to $\mathcal{D}$ (including $\mathcal{D}$ itself), hence the states of $\mathcal{D}''$ are quadruples of states over $P = \{p,q\}$ (the 4 bisimilar domains were provided in Figure~\ref{fig:all_bisim}). Hence, more precisely, the states of $\mathcal{D}''$ are quadruples $(s^1,\dots,s^4)$ where $s^i \in \mathcal{D}_i$, $i=1,\dots,4$, using the domain names introduced in   Figure~\ref{fig:all_bisim}. In particular, the initial state $s''_0$ of $\mathcal{D}''$ is the quadruple of the initial states of the 4 domains $\mathcal{D}_1,\dots,\mathcal{D}_4$, i.e., $s''_0 = (pq, pq, \neg pq, \neg pq)$. Note that $\mathcal{D}''$ is isomorphic to the original domain $\mathcal{D}$. This is different from the situation with the compatibility domain $\mathcal{D}'$ of $\mathcal{D}$, shown in Figure~\ref{figu:knockknock_compdomain}. The compatibility domain $\mathcal{D}'$ only has two states and represents $a$ as a non-deterministic action, hence is not isomorphic to $\mathcal{D}$. In other words, for this particular example, the behavioural equivalence domain gives a perfect representation of the original domain, whereas the compatibility domain does not. 

As mentioned earlier, the compatibility domain captures explicit knowledge and the behavioural equivalence domain captures implicit knowledge. In the door knocking domain, the explicit knowledge is very limited. Consider a learner that learns an event model $\mathcal{E}_a$
that is behaviourally correct with respect to explicit knowledge about $\mathcal{D}$. This means that the domain induced by $\mathcal{E}_a$ and $s_0 = pq$ is isomorphic to the compatibility domain $\mathcal{D}'$. Hence, consulting Figure~\ref{figu:knockknock_compdomain}, we would get $comp(s_0) \models K q \wedge \neg K [\mathcal{E}_a] q \land \neg K [\mathcal{E}_a] \neg q$. So, initially the agent explicitly knows $q$, but doesn't explicitly know whether knocking on the door will lead to another $q$-state or to a $\neg q$-state. 

Since the behavioural equivalence domain is isomorphic to the true domain, this confusion would not arise in the implicit knowledge of the agent. Consider an event model $\mathcal{E}'_a$ that is behaviourally correct with respect to implicit knowledge about $\mathcal{D}$. Then the domain induced by $\mathcal{E}'_a$ and $s_0 = pq$ is isomorphic to the behavioural equivalence domain $\mathcal{D}''$ of $\mathcal{D}$. Consulting Figure~\ref{fig:beh_equiv_real}, we can then conclude $s''_0 \models K q \wedge K [\mathcal{E}'_a] q \wedge K [\mathcal{E}'_a] [\mathcal{E}'_a] \neg q$. In other words, the agent implicitly knows $q$, and implicitly knows that knocking once doesn't change the truth value of $q$, but knocking twice does. 

We mentioned earlier that an agent would never be able to learn exactly how the knocking action affects $p$. However, it \emph{is} able to learn the possible ways it \emph{might} affect $p$ and still be consistent with all observation traces. This is encoded in the behavioural equivalence domain. Recall that, intuitively, $p$ is needed to keep track of how many times the door has been knocked at since it was last closed. Of course it is only relevant to distinguish between whether it was knocked at zero times or once since last closed, as after two knocks it is open again. This implies that we can use $p$ to encode the necessary information in four different ways: 1) $p$ becomes false when the door has been knocked at once since last closed; 2) $p$ becomes true when the door hasn't been knocked at since last closed; 3) $p$ becomes false when the door hasn't been knocked at since last closed; 4) $p$ becomes true when the door has been knocked at once since last closed. Consulting again Figure~\ref{fig:beh_equiv_real}, we can see that these four options are exactly the ones encoded by each of the elements of the quadruples: the first elements encode option 1), the second elements encode option 2), etc. In other words, the behavioural equivalence domain encodes exactly all the possible underlying dynamics of the domain (of which there are 4, corresponding to $\mathcal{D}_1,\dots,\mathcal{D}_4$). Furthermore, it represents these compactly as a single domain, by identifying the dynamics that are not distinguishable by any observation traces (how exactly $p$ encodes the book-keeping of number of knocks at the door, we can never observe).  

We can now also explain why the states of the behavioural equivalence domain are represented as tuples instead of sets of states. Consider what would happen if we replaced each of the tuples in Figure~\ref{fig:beh_equiv_real} by the set of elements contained in the tuple. Then both of the two upper states would become the set $\{ pq, \neg pq \}$. In other words, these two states would become identified, and we would get back to the compatibility domain! Tuples are needed to be able to keep track of which trace we are currently following, and hence to be able to distinguish states that are observationally indistinguishable, but not observationally trace equivalent.

\end{example}

In the next subsection, we will present a learning algorithm for implicit knowledge and prove that it produces a domain that is isomorphic to the behavioural equivalence domain, when presented with a certain set of observation traces for the initial state of the real domain. But before looking at this specific learner, let us discuss the issue of whether behavioural correctness with respect to implicit knowledge is achievable, in general, when presented with a \emph{finite} set of observation traces for the initial state of the real domain. Note that the set of observation traces for an initial state $s_0$ may be infinite. For example, an unknown domain with a single state $s_0$ and a single `loopy' action $a$, i.e.~$(s_0,a,s_0)\in T$ induces infinitely many observation traces. In fact, since we assume universal applicability of actions, \emph{every} domain has an infinite set of observation traces. However, does the learner \emph{need} to see all possible observation traces to achieve behavioural correctness? As we will soon show, the answer is no. A learner that wants to achieve behavioural correctness, i.e.~to produce a domain that is isomorphic to the behavioural equivalence domain of the real domain, can do so in finite time, from a finite set of observation traces. Borrowing some notions from formal learning theory, we can then say that the behavioural equivalence domain is \emph{finitely identifiable} \cite{mukouchi1992characterization, lange1992types}: the agent can conclusively learn it in finite time (up to isomorphism), given an appropriate set of observation traces. This is in contrast with \emph{identifiability in the limit} \cite{gold1967language}, which holds for any learning process in which the learner converges to the right solution after seeing finitely many inputs, but does not necessarily know at which point convergence takes place. In formal learning theory, a set that enables finite identification is called a \emph{definite finite tell-tale set} (DFTT for short, see \cite{mukouchi1992characterization,Gierasimczuk:2010aa,gierasimczuk2012complexity}, for an application in action learning see \cite{bolander2017learning}). We adapt this notion to our setting.

\begin{definition} Let $\mathcal{D} = ((S,A,T,s_0),\obs, \obsmod)$ be a domain over $P$. A set $\Sigma\subseteq \textnormal{ObsTr}(\mathcal{D},s_0)$ is called a \emph{definite finite tell-tale (DFTT) for} $\mathcal{D}$ if
\begin{enumerate}
    \item[(i)] $\Sigma$ is finite;
    \item[(ii)] for any domain $\mathcal{D}' = ((S',A,T',s'_0),\obs, \obsmod')$ over $P$, if $\Sigma\subseteq \textnormal{ObsTr}(\mathcal{D}',s'_0)$ then $\mathcal{D} \bisim \mathcal{D}'$.
\end{enumerate}
\end{definition}

As we show next, every domain $\mathcal{D}$ over $P$ has a DFTT, and such a DFTT is sufficient for learning a behaviourally correct domain with respect to $\mathcal{D}$ in finite time.

\begin{proposition}\label{prop:DFTT_exists} Every domain over $P$ has a DFTT.
\end{proposition}

\begin{proof} Let $\mathcal{D} = ((S,A,T,s_0),\obs, \obsmod)$ be a domain over $P$. Note that in condition (ii) for DFTTs, we quantify over domains over $P$ with the same set of actions and observations as $\mathcal{D}$, as bisimilarity is defined for domains with the same sets of actions and observations. As $P$, $A$ and $\obs$ are finite, this set of domains is finite. 
Let $\mathcal{D}_1,\dots, \mathcal{D}_n$ be an enumeration of all domains over $P$ with actions $A$ and observations $\obs$, and for each $\mathcal{D}_i$, define a function $Q_i: A^* \to \{0,1\}$ by 
\[Q_i(x) =  \begin{cases} 
      1 & \text{ if } \mathcal{D} \text{ and } \mathcal{D}_i \text{ produce the same observation trace based on } x \text{ from their initial states} \\
      0 & \text{otherwise} 
   \end{cases}
\]
Every domain $\mathcal{D}_i$ which is not bisimilar to $\mathcal{D}$ is also not trace equivalent to it (Lemma~\ref{lemma:bisim_trace}). Hence, the function $Q_i$ outputs $0$ at least on one of its inputs. For each $1\leq i\leq n$ for which $Q_i$ outputs $0$ at least once, pick one $x_i$ with $Q_i(x_i)=0$, and let $\tau_i$ be the observation trace produced by $\mathcal{D}$ based on $x_i$. Define the set $\Sigma$ consisting of all such $\tau_i$. We claim that $\Sigma$ is a DFTT. For condition (i): $\Sigma$ is clearly finite, since the $\tau_i$ range over $1\leq i < n$. For condition (ii), take a domain $\mathcal{D}'$ and suppose that $F \subseteq ObsTr(\mathcal{D}',s'_0)$. For a contradiction, suppose that $\mathcal{D}'$ is not bisimilar to $\mathcal{D}$. $\mathcal{D}'=\mathcal{D}_i$, for some $i\in \{1,\dots,n\}$. Since $\mathcal{D}_i$ is not bisimilar to $\mathcal{D}$, by construction of $\Sigma$, $\tau_i\not\in ObsTr(\mathcal{D}_i,s'_0)$. Hence $\Sigma\not\subseteq ObsTr(\mathcal{D}_i,s'_0)$, which gives a contradiction.
\end{proof}

Proposition \ref{prop:DFTT_exists} shows that a DFTT exists for each domain $\mathcal{D}$ over $P$. But we can in fact do better and \emph{bound} the length of the observation traces in a DFTT.

\begin{proposition}\label{prop:DFTT_size} Every domain over $P$ has a DFTT consisting of all observation traces that have $2^{2|P|}$ actions.
\end{proposition}
\begin{proof}
Let $\mathcal{D}_1 = ((S^1,A,T^1,s^1_0),\obs, \obsmod^1)$ be a domain over $P$. Note that in condition (ii) for DFTTs, we quantify over domains over $P$ with the same set of actions and observations as $\mathcal{D}_1$, as bisimilarity is defined for domains with the same sets of actions and observations. For each domain over $P$ of the form $\mathcal{D}_2 = ((S^2,A,T^2,s^2_0),\obs, \obsmod^2)$, we can construct a \emph{product} domain $\mathcal{D} = ((S,A,T,s_0),\obs, \obsmod)$  where $S=S^1\times S^2$, for all $(s,s')\in S$ and $a\in A$, $T((s,s'),a)=(T^1(s,a),T^2(s',a))$, $Obs((s,s'))=(Obs(s),Obs(s'))$ and $s_0 = (s^1_0,s^2_0)$. If $\mathcal{D}_1$ and $\mathcal{D}_2$ are trace inequivalent, there will be a trace in $\mathcal{D}$ leading to a state $(s,s')$ with $Obs(s) \neq Obs(s')$. Now note that the number of states in $\mathcal{D}$ is $|S^1|\times|S^2| \leq 2^{|P|}\cdot 2^{|P|} = 2^{2|P|}$. 
If $(s,s')$ is reachable, then it is reachable by a simple path, i.e.~a sequence of alternating states and actions that does not visit the same state twice~\cite[Th~6]{chartrand2012a}. Hence, we need not consider paths longer than that--- if a state $(s,s')$ with $Obs(s) \neq Obs(s')$ is reachable, it will be reachable by a simple path, and hence a path of length at most $2^{2|P|}$. 
\end{proof}

Having established the existence of DFTTs of a bounded finite size, we can now show that domains can be identified, up to bisimilarity, in finite time.

\begin{proposition}\label{prop:learnable} If there is a DFTT for $\mathcal{D}$ over $P$, then it is possible to learn a domain that is behaviourally correct with respect to implicit knowledge in finite time.
\end{proposition}

\begin{proof}
Let $\Sigma$ be a DFTT for $\mathcal{D}$. 

\medskip
\noindent \emph{Claim 1}. $\Sigma$ features all actions and observations in $\mathcal{D}$.

\smallskip
\noindent \emph{Proof of Claim 1.} 
For contradiction, suppose that there is an action $a\in A$ which does not occur in any observation trace of $\Sigma$. Define a domain $\mathcal{D}'$ with $S'=S$, $A'=A$,  $s_0=s'_0$ and $T'(s,a')=T(s,a')$ for all $a\in A - \{a\}$. Let $T(s_0,a)=t$, and choose an observation for $t$ in $\mathcal{D}'$ so that $Obs'(t)\neq Obs(T(s_0,a))$. Then $\mathcal{D}$ and $\mathcal{D}'$ have different observation traces, and thus are not bisimilar, but $\Sigma\subseteq \textnormal{ObsTr}(\mathcal{D}',s'_0)$. This gives contradiction.
Now suppose that there is some observation $o\in\obs$ not featured in $\Sigma$. Let $s\in S$ be some state with $Obs(s)=o$. Define $\mathcal{D}'$ as $\mathcal{D}$ except for $Obs(s)\neq Obs'(s)$. Since $o$ does not occur in $\Sigma$, $\Sigma \subseteq \textnormal{ObsTr}(\mathcal{D}',s'_0)$, which gives contradiction. This completes the proof of Claim 1.
\medskip

From Claim 1, we know that all actions and observations from $\mathcal{D}$ appear in $\Sigma$. Let $\mathcal{D}_1,\dots, \mathcal{D}_n$ be an enumeration of all domains over $P$ with actions $A$ and observations $\obs$. 
For each domain $\mathcal{D}'$ over $P$, check if $\Sigma\subseteq \{ \tau \in \textnormal{ObsTr}(\mathcal{D}',s'_0) \mid \tau \text{ has length at most } 2^{2|P|}
\}$. Any domain passing this check is, by definition, bisimilar to $\mathcal{D}$. As the set of domains over $P$ with actions $A$ and observations $\obs$ is finite, and $\Sigma$ is a finite set of finite sequences (Proposition \ref{prop:DFTT_exists}), this check can be done in finite time. Once we have computed the set $\mathbb{D}$ of domains passing the check, we can construct the synchronous composition of the domains in $\mathbb{D}$, following Definition
\ref{definition:behav_equiv_domain}. As $\mathbb{D}$ is the set of domains bisimilar to $\mathcal{D}$, and the behavioural equivalence domain induced by $\mathcal{D}$ is simply the synchronous composition of the domains in $\mathbb{D}$, we will have constructed, from $\mathbb{D}$, the behavioural equivalence domain induced by $\mathcal{D}$. 
\end{proof}

Using the bound for the size of a DFTT established in Proposition \ref{prop:DFTT_size}, we can define a notion of a \emph{sound and complete} set of observation traces for a domain.

\begin{definition} Let $\mathcal{D} = ((S,A,T,s_0),\obs, \obsmod)$. A \emph{sound and complete} set of observation traces $\Sigma$ for $\mathcal{D}$ is a subset of $\textnormal{ObsTr}(\mathcal{D},s_0)$ such that every observation trace from $s_0$ with $2^{2|P|}$ actions is in $\Sigma$.
\end{definition}

\medskip

\subsection{A behaviourally correct learner of implicit knowledge}

We now present our second learner of the paper and prove that it achieves its goal (producing a behaviourally correct domain) when presented with a sound and complete set of observation traces from the initial state of the real domain. The learner first computes, in an incremental fashion, the set $\mathbb{D}$ of domains bisimilar to the real domain $\mathcal{D}$. It then computes and outputs the synchronous composition of the domains in $\mathbb{D}$, which is, by definition, the same as the behavioural equivalence domain induced by $\mathcal{D}$.

To describe the learner, we need the following additional technical definition. A \emph{history for observation trace} $(o_0,a_0,o_{1},\dots, a_n, o_{n+1})$ is a sequence $h = (\bar{s}_0, o_0, a_1,\bar{s}_{1},o_{1}, \dots, a_n, s_{n+1}, o_{n+1})$ where $\bar{s}_i\in comp(o_i)$ for $0\leq i \leq n+1$. Intuitively, $h$ \emph{hypothesises} the state $\bar{s}_i$ that gave rise to the observations $o_i$. Since observations are non-noisy, it is only relevant to consider hypothesised states that are compatible with the observations.

By choosing such states to explain $(o_0,a_0,o_{1},\dots, a_n, o_{n+1})$, the history $h$ induces a possible domain that could have generated it.

\begin{definition} Let $h = (\bar{s}_0, o_0, a_1,\bar{s}_{1},o_{1}, \dots, a_n, s_{n+1}, o_{n+1})$ be a history (for some observation trace). 
We then define the following notations
\begin{itemize}
    \item $\states^h = \{ \bar{s}_i \mid i=0,\dots, n+1\}$
    \item $\actions^h  =  \{ a_i \mid i=0,\dots, n\}$
    \item $\transmod^h  =  \{(\bar{s}_i,a_i,\bar{s}_{i+1}) \mid i=0,\dots,n\}$
    \item $s^h_0 = \bar{s}_0$
    \item $\obs^h  = \{ o_i \mid i=0,\dots,n+1\}$
    \item $\obsmod^h(\bar{s}_i) = o_i$, for $i=0,\dots,n+1$.
\end{itemize}
The \emph{domain induced by} $h$ is the domain $\mathcal{D}^h$ given by $\mathcal{D}^h= ((\states^h,\actions^h,\transmod^h, s^h_0), \obs^h, \obsmod^h)$.
\end{definition} 

The learner has three related components, described in Algorithms \ref{algo:histories} and \ref{algo:L1}, and in Theorem~\ref{cor:L_1}. The procedure in Theorem~\ref{cor:L_1} builds on the output of Algorithm \ref{algo:L1}, which in turn builds on the output of Algorithm \ref{algo:histories}. Algorithm \ref{algo:histories}, called $\textsc{Histories}(P,\tau)$, generates a set of histories from a set of proposition symbols $P$ and a single observation trace $\tau$. The algorithm generates the set of all histories for $\tau$, assuming the domain is described by propositional symbols $P$. Each history $h$ produced by this algorithm on input $\tau$ induces a domain $\mathcal{D}^h$ over $P$ that mimics 
the interaction seen in $\tau$. However, as $\tau$ may not exhibit \textit{every} possible interaction with the real domain $\mathcal{D}$, $\mathcal{D}^h$ provides a partial description of $\mathcal{D}$. Intuitively, such partial descriptions could be carefully `stitched together' to generate complete domain descriptions bisimilar to 
$\mathcal{D}$. Algorithm \ref{algo:L1} presents such a procedure,  $\textsc{Domains}(P,\Sigma)$. Given a set of observation traces $\Sigma = \{\tau_1,\dots,\tau_n\}$, it iteratively calls   $\textsc{Histories}(P,\tau_i)$ for $1\leq i\leq n$. Each call produces a set of histories $\mathcal{H}_i$ for $\tau_i$. $\textsc{Domains}(P,\Sigma)$ then checks each set of histories of the form $H=\{h_1,\dots,h_n \mid h_i=(s,\dots), h_i\in \mathcal{H}_i,1\leq i \leq n, \text{ for some state }s\}$, and constructs a domain $\mathcal{D}^H$ by taking the union of the components of the domains $\mathcal{D}^{h_i}$, component-wise. If $\mathcal{D}^H$ is deterministic, the domain is added to a set $\mathbb{D}$. The set of domains $\mathbb{D}$ is the output of $\textsc{Domains}(P,\Sigma)$. Intuitively, each $\mathcal{D}^H\in\mathbb{D}$ is a domain that mimics all observation traces  in the input, without breaking the condition of determinism for domains. Lemma \ref{theorem:identif_trace} shows that $\textsc{Domains}(P,\Sigma)$ produces all and only the domains that are bisimilar to the real domain $\mathcal{D}$, when $\Sigma$ is a sound and complete set of observation traces for $\mathcal{D}$. The overall learner, described in Theorem~\ref{cor:L_1}, then returns the synchronous composition of the domains $\mathbb{D}$ produced by $\textsc{Domains}(P,\Sigma)$. Since $\mathbb{D}$ is the set of domains bisimilar to $\mathcal{D}$, the learner will return the behavioural equivalence domain induced by $\mathcal{D}$, and hence be behaviourally correct with respect to implicit knowledge about $\mathcal{D}$.



\begin{algorithm}
\SetKwInOut{Input}{Input}\SetKwInOut{Output}{Output}
\SetAlgoLined
\Input{$P$ (propositional symbols), $\tau=(o_0,a_0,o_1,\dots, o_n,a_n,o_{n+1})$ (observation trace).} 
\Output{$H_n$ (histories).}
\For{$i=0,\dots, n$ 
} 
    {
    $Con_i \coloneqq \{ (\bar{s}_i,\bar{s}_{i+1}) \in 2^P\times 2^P \mid \bar{s}_i \in comp(o_i), \bar{s}_{i+1} \in comp(o_{i+1}) \}$\; 
    $H_i = \emptyset$\; 
    \For{$(\bar{s}_{i},\bar{s}_{i+1})\in Con_i$}
        {
        \If{$i=0$}
            {
            Start a history $h = (\bar{s}_{0},o_0,a_0, \bar{s}_{1},o_{1})$\; 
             \If{$\mathcal{D}^h$ is deterministic}
                 {$H_0 \coloneqq H_0 \cup \{h\}$\;}}
        
        \Else{
                \For{$h=(\bar{s}_{0},o_0,\dots,\bar{s}_{i},o_i) \in H_{i-1}$}
                {$h' \coloneqq (\bar{s}_{0},o_0,\dots,\bar{s}_{i},o_i,a_i,\bar{s}_{i+1},o_{i+1})$\;
                \If{$\mathcal{D}^{h'}$ is deterministic}{$H_i \coloneqq H_i \cup \{ h'\}$\;} 
                }
            }
        }
    }
 \Return{$H_n$}\;
 \caption{$\textsc{Histories}(P,\tau)$}
 \label{algo:histories}
\end{algorithm}

\begin{algorithm}
\SetKwInOut{Input}{Input}\SetKwInOut{Output}{Output}
\SetAlgoLined
\Input{$P$ (propositional symbols), $\Sigma = \{\tau_1,\dots,\tau_n \}$ (set of observation traces).}
\Output{$\mathbb{D}$ (set of domains over $P$).}
$\mathbb{D}\coloneqq \emptyset$\;
\For{$i=1,\dots, n$ 
} 
    {
    $\mathcal{H}_i \coloneqq \textsc{Histories}(P,\tau_i)$\; 
    }
\For{$h_1\in \mathcal{H}_1,\dots, h_n\in \mathcal{H}_n$
} 
    {
    \If{$h_1,\dots, h_n$ all start in the same state $s_0$}{
    
    $H\coloneqq \{h_1,\dots,h_n\}$\;
    $s^H_0\coloneqq s_0$\;
    
    \scalebox{0.97}[1]{$\mathcal{D}^H \coloneqq ((\bigcup^k_{i=1}S^{h_i}, \bigcup^k_{i=1}A^{h_i}, \bigcup^k_{i=1}T^{h_i},s^H_0), \bigcup^k_{i=1}\obs^{h_i}, \bigcup^k_{i=1}Obs^{h_i})$}\;
    
    \If{$\mathcal{D}^{H}$ is deterministic}{
    $\mathbb{D} \coloneqq \mathbb{D} \cup \{\mathcal{D}^{H}\}$
    }
    } 
}

 \Return{$\mathbb{D}$}\;
 \caption{$\textsc{Domains}(P,\Sigma)$}
 \label{algo:L1}
\end{algorithm}

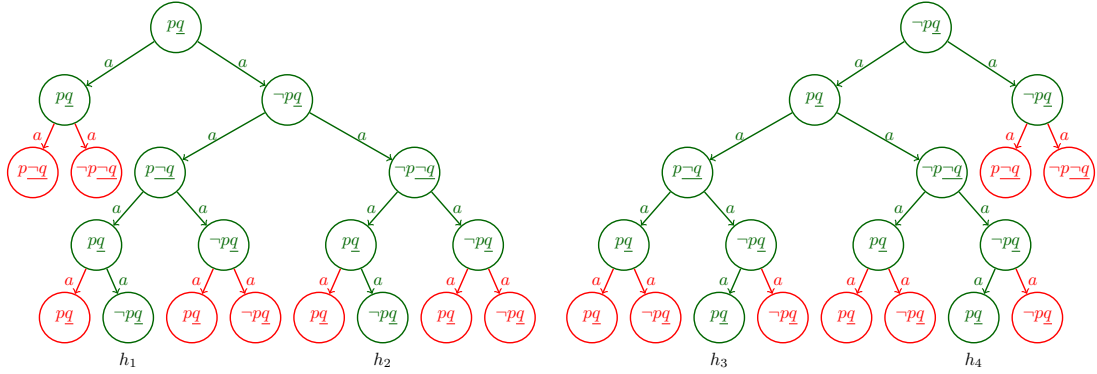
\begin{figure}
\begin{center}
    \scalebox{0.67}{
                \begin{forest}
                sn edges
                [, phantom, s sep = 0.6cm
                    [\both, mygreen, name=level0
                        [\both, mygreen, edge=mygreen, name=level10, edge label = {node [midway,left] {$a$} }
                            [\onlyP, red, edge=red, edge label = {node [midway,left] {$a$} }, name=level20]
                            [\neither, red, edge=red, edge label = {node [midway, right] {$a$} }, name=level21] 
                        ]
                        [\onlyQ, mygreen, edge=mygreen, edge label = {node [midway,right] {$a$}}, name=level11
                            [\onlyP, mygreen, edge=mygreen, edge label = {node [midway,left] {$a$}},name=level22  [\both, mygreen, edge=mygreen,name=h1, edge label = {node [midway,left] {$a$}} [\both, red, edge=red, edge label = {node [midway,left] {$a$}}] [\onlyQ, mygreen, edge=mygreen, edge label = {node [midway,right] {$a$}},name=h1new]] [\onlyQ, mygreen, edge=mygreen, name=h2, edge label = {node [midway,right] {$a$}} [\both, red, edge=red, edge label = {node [midway,left] {$a$}}] [\onlyQ, red, edge=red, edge label = {node [midway,right] {$a$}}] ]] 
                            [\neither, mygreen, edge=mygreen , edge label = {node [midway,right] {$a$}}, name=level23 [\both, mygreen, edge=mygreen,name=h3new1, edge label = {node [midway,left] {$a$}} [\both, red, edge=red, edge label = {node [midway,left] {$a$}}] [\onlyQ, mygreen, edge=mygreen, edge label = {node [midway,right] {$a$}}, name=h2new] ] [\onlyQ, mygreen, edge=mygreen,name=h4, edge label = {node [midway,right] {$a$}} [\both, red, edge=red, edge label = {node [midway,left] {$a$}}] [\onlyQ, red, edge=red, edge label = {node [midway,right] {$a$}}]]] 
                        ]   
                    ]
                    [\onlyQ, mygreen, name=level00
                        [\both, mygreen, edge=mygreen,edge label = {node [midway,left] {$a$}}, name=level12
                            [\onlyP, mygreen, edge=mygreen,edge label = {node [midway,left] {$a$}}, name=level24 [\both, mygreen, edge=mygreen,name=h5,edge label = {node [midway,left] {$a$}} [\both, red, edge=red, edge label = {node [midway,left] {$a$}}] [\onlyQ, red, edge=red, edge label = {node [midway,right] {$a$}}]] [\onlyQ, mygreen, edge=mygreen,name=h6,edge label = {node [midway,right] {$a$}} [\both, mygreen, edge=mygreen, edge label = {node [midway,left] {$a$}},name=h3new] [\onlyQ, red, edge=red, edge label = {node [midway,right] {$a$}}] ]] 
                            [\neither, name=level25, mygreen, edge=mygreen,edge label = {node [midway,right] {$a$}} [\both, mygreen, edge=mygreen,name=h7,edge label = {node [midway,left] {$a$}} [\both, red, edge=red, edge label = {node [midway,left] {$a$}}] [\onlyQ, red, edge=red, edge label = {node [midway,right] {$a$}}]] [\onlyQ, mygreen, edge=mygreen,name=h8,edge label = {node [midway,right] {$a$}} [\both, mygreen, edge=mygreen, name=h4new, edge label = {node [midway,left] {$a$}}] [\onlyQ, red, edge=red, edge label = {node [midway,right] {$a$}}]]] 
                        ]
                        [\onlyQ, mygreen, edge=mygreen,edge label = {node [midway,right] {$a$}}, name=level13
                            [\onlyP,name=level27, red, edge=red,edge label = {node [midway,left] {$a$}}] 
                            [\neither,name=level28, red, edge=red ,edge label = {node [midway,right] {$a$}}] 
                        ]
                    ]
                ]
                \node[below=0.05cm of h1new](n1){$h_1$}; 
                \node[below=0.05cm of h2new](n2){$h_2$}; 
                \node[below=0.05cm of h3new](n3){$h_3$}; 
                \node[below=0.05cm of h4new](n4){$h_4$}; 
                \end{forest}
    }
    \end{center}
\caption{A visualisation of the execution of $\textsc{Histories}(P,\tau)$ (Algorithm \ref{algo:histories}) with input $P=\{p,q\}$ and the observation trace $\tau= (q,a,q,a,\neg q,a,q,a,q)$ from the door knocking domain. Each branch of each tree represents an initial segment of a history for $\tau$. The leftmost branch of the leftmost tree represents the history $(pq,q,a,pq,q,a,p\neg q,\neg q)$: The root is the state $pq$ in which $q$ is observed; then $a$ is executed leading again to the state $pq$ where $q$ is observed; finally, $a$ is executed again, this time leading to the state $p\neg q$ where $\neg q$ is observed.} \label{fig:histories}
\end{figure}
\begin{example}
Figure \ref{fig:histories} shows the behaviour of $\textsc{Histories}(P,\tau)$ (Algorithm \ref{algo:histories}) when executed with $P=\{p,q\}$ and the observation trace $\tau= (q,a,q,a,\neg q,a,q,a,q)$ from the door knocking domain. The algorithm iterates over $i=0,1,2,3$. At step $0$, it generates the set $Con_0 = \{(\bar{s}_0,\bar{s}_1) \mid \bar{s}_0,\bar{s}_1\in comp(q) \}=\{(pq,pq),(pq,\neg p q),(\neg pq, pq), (\neg pq, \neg pq)\}$. The algorithm starts a history $(\bar{s}_0,q,a,\bar{s}_1,q)$ for each such pair, as they all induce deterministic domains. Each path of length $1$ starting at the root of a tree corresponds to one such history. At step 1, the algorithm computes $Con_1$. For each pair $(\bar{s}_1,\bar{s}_2)\in Con_1$, it considers each history from step $0$ whose last state is $\bar{s}_1$ and tries to extend it with $(a, \bar{s}_2,\neg q)$. From the left child of the root of the left tree, two history extensions are tried and discarded (shown in red). The first one would extend $(pq, q, a, pq, q)$ into $(pq, q, a, pq, q, a, p\neg q, \neg q)$. This extension is discarded because it would induce the non-deterministic transition function $T(pq,a) = \{pq, p\neg q\}$. The second one would extend $(pq, q, a, pq, q)$ into $(pq, q, a, pq, q, a, \neg p\neg q, \neg q)$. This is discarded as well, as it would induce the non-deterministic transition function $T(pq,a) = \{pq, \neg p\neg q\}$. On the right tree, the history extensions marked in red are also discarded, as they would similarly induce non-deterministic transition functions. The remaining extensions, marked in green, are accepted. In step $2$, the histories generated in step $1$ are again extended. In step $3$, several history extensions are discarded, since they would yield non-deterministic domains. The algorithm then outputs the set of histories $\{h_1,\dots, h_4\}$, each corresponding to a path from one root of a tree to a leaf of that tree.

\begin{figure} 
\begin{minipage}{0.5\textwidth}
 \[
 \scalebox{0.75}{
 \begin{tikzpicture}[auto,align=center]
  \node[circle,draw,inner sep=-10pt,minimum size=1.5cm,label={below:}] (s0) at (0,0) {
 \begin{tikzpicture}[auto,align=center,inner sep=0pt] 
   \node (s01) {$s_0: p \underline{q}$};
  \end{tikzpicture}
  };
    \node[right of=s0,node distance=40mm,circle,draw,inner sep=-10pt,minimum size=1.5cm,label={below:}] (s3) {
    \begin{tikzpicture}[inner sep=0pt]
       \node (s21) {$\neg p \underline{q}$};
    \end{tikzpicture}
    };
    \path[-latex] (s0) edge node {$a$} (s3);
   \node[below right of=s0,node distance=28mm,circle,draw,inner sep=-10pt,minimum size=1.5cm,label={below:}] (s2) {
    \begin{tikzpicture}[inner sep=0pt]
       \node (s21) {$p \underline{\neg q}$};
    \end{tikzpicture}
    };
    \path[-latex] (s3) edge node {$a$} (s2);
    \path[-latex] (s2) edge node {$a$} (s0);    \node[left of=s2,xshift=-9mm] {$\mathcal{D}^{\{h_1\}}$};
 \end{tikzpicture}
 }
 \]
  \end{minipage} \hfill 
\begin{minipage}{0.5\textwidth}
\[
\scalebox{0.75}{
\begin{tikzpicture}[auto,align=center]
\node[circle,draw,inner sep=-10pt,minimum size=1.5cm,label={below:}] (s0) at (0,0) {
\begin{tikzpicture}[auto,align=center,inner sep=0pt] 
\node (s01) {$s_0: p \underline{q}$};
\end{tikzpicture}
};
\node[right of=s0,node distance=40mm,circle,draw,inner sep=-10pt,minimum size=1.5cm,label={below:}] (s3) {
\begin{tikzpicture}[inner sep=0pt]
   \node (s21) {$\neg p \underline{q}$};
\end{tikzpicture}
};
\path[-latex] (s0) edge node {$a$} (s3);
\node[below right of=s0,node distance=28mm,circle,draw,inner sep=-10pt,minimum size=1.5cm,label={below:}] (s2) {
\begin{tikzpicture}[inner sep=0pt]
   \node (s21) {$\neg p \underline{\neg q}$};
\end{tikzpicture}
};
\path[-latex] (s3) edge node {$a$} (s2);
\path[-latex] (s2) edge node {$a$} (s0);
\node[left of=s2,xshift=-9mm] {$\mathcal{D}^{\{h_2\}}$};
\end{tikzpicture}
}
\]
\end{minipage} 
\begin{minipage}{0.5\textwidth}
 \[
 \scalebox{0.75}{
 \begin{tikzpicture}[auto,align=center]
  \node[circle,draw,inner sep=-10pt,minimum size=1.5cm,label={below:}] (s0) at (0,0) {
 \begin{tikzpicture}[auto,align=center,inner sep=0pt] 
   \node (s01) {$s_0: \neg p \underline{q}$};
  \end{tikzpicture}
  };
    \node[right of=s0,node distance=40mm,circle,draw,inner sep=-10pt,minimum size=1.5cm,label={below:}] (s3) {
    \begin{tikzpicture}[inner sep=0pt]
       \node (s21) {$p \underline{q}$};
    \end{tikzpicture}
    };
    \path[-latex] (s0) edge node {$a$} (s3);
   \node[below right of=s0,node distance=28mm,circle,draw,inner sep=-10pt,minimum size=1.5cm,label={below:}] (s2) {
    \begin{tikzpicture}[inner sep=0pt]
       \node (s21) {$p \underline{\neg q}$};
    \end{tikzpicture}
    };
    \path[-latex] (s3) edge node {$a$} (s2);
    \path[-latex] (s2) edge node {$a$} (s0);    \node[left of=s2,xshift=-9mm] {$\mathcal{D}^{\{h_3\}}$};
 \end{tikzpicture}
 }
 \] 
  \end{minipage} \hfill 
\begin{minipage}{0.5\textwidth}
 \[
 \scalebox{0.75}{
 \begin{tikzpicture}[auto,align=center]
  \node[circle,draw,inner sep=-10pt,minimum size=1.5cm,label={below:}] (s0) at (0,0) {
 \begin{tikzpicture}[auto,align=center,inner sep=0pt] 
   \node (s01) {$s_0: \neg p \underline{q}$};
  \end{tikzpicture}
  };
    \node[right of=s0,node distance=40mm,circle,draw,inner sep=-10pt,minimum size=1.5cm,label={below:}] (s3) {
    \begin{tikzpicture}[inner sep=0pt]
       \node (s21) {$p \underline{q}$};
    \end{tikzpicture}
    };
    \path[-latex] (s0) edge node {$a$} (s3);
   \node[below right of=s0,node distance=28mm,circle,draw,inner sep=-10pt,minimum size=1.5cm,label={below:}] (s2) {
    \begin{tikzpicture}[inner sep=0pt]
       \node (s21) {$\neg p \underline{\neg q}$};
    \end{tikzpicture}
    };
    \path[-latex] (s3) edge node {$a$} (s2);
    \path[-latex] (s2) edge node {$a$} (s0);    \node[left of=s2,xshift=-9mm] {$\mathcal{D}^{\{h_4\}}$};
 \end{tikzpicture}
 }
 \] 
  \end{minipage} 
\caption{The domains $\mathcal{D}^{\{h_1\}},\dots,\mathcal{D}^{\{h_4\}}$ outputted by $\textsc{Domains}(P,\{\tau\})$ (Algorithm \ref{algo:L1}) with input $P=\{p,q\}$ and the observation trace $\tau= (q,a,q,a,\neg q,a,q,a,q)$ from the door knocking domain.}
\label{fig:combine}
\end{figure}
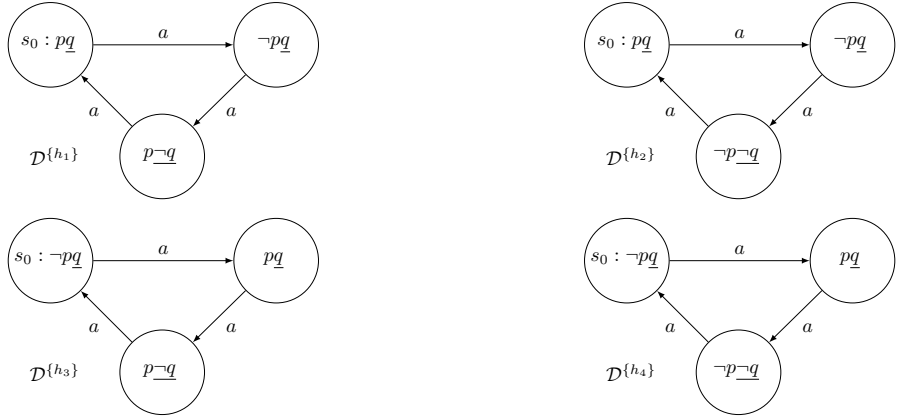
Figure \ref{fig:combine} shows the output of $\textsc{Domains}(P,\Sigma)$ (Algorithm \ref{algo:L1}) when executed with $P=\{p,q\}$ and $\Sigma = \{\tau\}$. The algorithm first calls $\textsc{Histories}(P,\tau)$, which as just seen, generates the set of histories $\{h_1,\dots, h_4\}$. The algorithm then generates one domain $\mathcal{D}^{\{h_i\}}$ for each $1\leq i\leq n$. For instance $\mathcal{D}^{\{h_1\}}$ is generated by the history $h_1 = (pq,q,a,\neg pq,q,a,p\neg q,\neg q,a,pq,q,a,\neg p q, \neg p)$ from Figure~\ref{fig:histories}. The domain $\mathcal{D}^{h_1}$ is simply read off the history $h_1$: the history tells us that the $a$-successor of the state $pq$ is the state $\neg pq$, that the successor of $\neg pq$ is $p\neg q$, and that, finally, the successor of $p\neg q$ is the original state $pq$. This immediately produces the domain $\mathcal{D}^{\{h_1\}}$ shown in the top left of Figure~\ref{fig:combine}, where we also underlined the observations. These observations are similarly read directly off $h_1$. Note that, for all $i$,  $\mathcal{D}^{\{h_i\}}$ is the same domain as $\mathcal{D}_{i}$ from Figure \ref{fig:all_bisim}, which depicted all domains bisimilar to the door knocking domain. The algorithm has thus produced exactly the set of domains bisimilar to the door knocking domain.

Finally, the overall learner, on input $P=\{p,q\}$ and $\Sigma=\{\tau\}$, takes the result of $\textsc{Domains}(P,\{\tau\})$, i.e.~$\{\mathcal{D}^{\{h_1\}},\dots,\mathcal{D}^{\{h_4\}}\}$, and returns the synchronous composition of $\{\mathcal{D}^{\{h_1\}},\dots,\mathcal{D}^{\{h_4\}}\}$. As $\{\mathcal{D}^{\{h_1\}},\dots,\mathcal{D}^{\{h_4\}}\}= \{\mathcal{D}_1,\dots,\mathcal{D}_4\}$, the output of the learner is in fact the behavioural equivalence domain induced by the door knocking domain.
\end{example}

\begin{lemma}\label{theorem:identif_trace} 
With input $P$ and a sound and complete set $\Sigma$ of observation traces for $\mathcal{D}$, algorithm $\textsc{Domains}(P,\Sigma)$ (Algorithm~\ref{algo:L1}) 
 returns the set of all domains over $P$ that are bisimilar to $\mathcal{D}$. 
\end{lemma}

\begin{proof} Throughout the proof, let $\tau_1,\dots, \tau_k$ be an enumeration of $\Sigma$, let $\mathcal{H}_i$ denote the output of $\textsc{Histories}(P,\tau_i)$, and denote each $\tau_i$ as follows: \[\tau_i = (Obs(s^i_0),a^i_0,Obs(s^i_{1}), \dots, a^i_n, Obs(s^i_{n+1})).\]
The execution trace associated with $\tau_i$ will be denoted $\varepsilon_i$.

($\subseteq$) We show first that $\textsc{Domains}(P,\Sigma)\subseteq \{\mathcal{D}' \mid \mathcal{D}\bisim \mathcal{D}'\}$. Let $\mathcal{D}^H \in \textsc{Domains}(P,\Sigma)$, $H=\{h_1,\dots, h_k\}$. We refer to the components of $D^H$ as $S^H$, $A^H$, etc. We show that $\mathcal{D}^H \bisim \mathcal{D}$. Note that each $h_i\in H$ is a history for $\tau_i$, i.e.~$h_i\in \mathcal{H}_i$. Note that $h_i$ can thus be written as follows:
\[h_{i}= (\bar{s}^i_0,Obs(s^i_0),a^i_0,\bar{s}^i_1,Obs(s^i_{1}), \dots, a^i_n,\bar{s}^i_{n+1}, Obs(s^i_{n+1}))\]
where each $\bar{s}^i_j$ is some state from $comp(o^i_j)$.
Define a relation $Z\subseteq S\times S^H$ by \[\text{$s Z s'$  iff $s=s^i_j$ and $s'=\bar{s}^i_j$ for some $i\in \{1,\dots,k\}$ and some $j\in \{0,\dots,n\}$.}\]
We claim that $Z$ is a bisimulation between $\mathcal{D}$ and $\mathcal{D}^h$, 
i.e.~$Z$ satisfies (i) $s_0Zs^H_0$, (ii) Observational indistinguishability, (iii) Forth, and (iv) Back.
\begin{description}
    \item[$\textnormal{(i)}$] Since each $\tau_i\in \textnormal{ObsTr}(\mathcal{D},s_0)$, $s^1_0=s^2_0=\dots=s^k_0=s_0$. By definition of $\mathcal{D}^H$, $\bar{s}^1_0=\bar{s}^2_0=\dots = \bar{s}^k_0 = s^H_0$ (line 8, Algorithm \ref{algo:L1}). Thus, by definition of $Z$, $s_0 Z s^H_0$.
    \item[$\textnormal{(ii)}$] Let $s^i_jZ\bar{s}^i_j$. From lines 6 and 11 of Algorithm \ref{algo:histories} and the definition of $\mathcal{D}^{h_i}$, it follows that $Obs(s^i_j)=Obs^{h_i}(\bar{s}^i_j)$. Since $\mathcal{D}^H$ is deterministic (line 10, Algorithm \ref{algo:L1}) $Obs^H(\bar{s}^i_j)=Obs^{h_i}(\bar{s}^i_j)$, since $Obs^H = \bigcup_i Obs^{h_i}$ (line 9, Algorithm \ref{algo:L1}). Hence $Obs(s^i_j)=Obs^H(\bar{s}^i_j)$.
    \item[$\textnormal{(iii)}$] Suppose that $sZs'$ and $T(s,a)=t$. As $sZs'$, $s=s^i_j$ and $s'=\bar{s}^i_j$ for some $i\in\{1,\dots, k\}$ and some $j\in\{0,\dots, n\}$. 
    Consider an execution trace of the form
    \[\varepsilon = (s_0,a^i_0,\dots, a^i_{j-1},s, a, t, \dots)\]
   with $n$ actions. Note that $\varepsilon\in\textnormal{Tr}(\mathcal{D},s_0)$, so there is some $\ell\in\{1,\dots,k\}$ such that $\varepsilon=\varepsilon_\ell$. It thus follows that $t=s^\ell_{j+1}$ and that $\tau_\ell \in \textnormal{ObsTr}(\mathcal{D},s_0)$. 
   
    \noindent \emph{Claim 1.} $\bar{s}^i_m = \bar{s}^\ell_m$ implies $\bar{s}^i_{m+1} = \bar{s}^\ell_{m+1}$, for $0\leq m\leq j-1$.
    
    \smallskip
    Suppose for contradiction that for some $m$, $\bar{s}^i_m = \bar{s}^\ell_m$ but $\bar{s}^i_{m+1} \neq \bar{s}^\ell_{m+1}$. From lines 6 and 11 of Algorithm \ref{algo:histories}, we can see that then $(\bar{s}^i_m,a^i_m,\bar{s}^i_{m+1})\in T^{h_i}$ and $(\bar{s}^\ell_m,a^i_m,\bar{s}^\ell_{m+1})\in T^{h_\ell}$. By definition of $\mathcal{D}^H$ (line 9, Algorithm \ref{algo:L1}),
    $(\bar{s}^i_m,a^i_m,\bar{s}^i_{m+1}),(\bar{s}^\ell_m,a^i_m,\bar{s}^\ell_{m+1})\in T^H$. But then, since $\bar{s}^i_{m+1} \neq \bar{s}^\ell_{m+1}$,  $T^H$ is not deterministic, which gives contradiction. This completes the proof of Claim 1.
    
    \medskip
    
    Note that $\bar{s}^i_0=\bar{s}^\ell_0 = s^H_0$, by definition of $\mathcal{D}^H$. From Claim 1, we then get $\bar{s}^i_j = \bar{s}^\ell_j$. From lines 6 and 11 of Algorithm \ref{algo:histories}, we get that $T^{h_\ell}(\bar{s}^i_j,a)=\bar{s}^\ell_{j+1}$. By definition of $\mathcal{D}^H$, $T^{H}(\bar{s}^i_j,a)=\bar{s}^\ell_{j+1}$. As $s'=\bar{s}^i_j$ and $t=s^\ell_{j+1}$, letting $t'=\bar{s}^\ell_{j+1}$ we get: there is a $t'\in S^H$ s.t. $T^H(s',a)=t'$ and $t Z t'$.

    \item[$\textnormal{(iv)}$] Let $sZs'$ and $T^H(s',a)=t'$. As $sZs'$, $s=s^i_j$ and $s'=\bar{s}^i_j$ for some $i\in\{1,\dots, k\}$ and some $j\in\{0,\dots, n\}$. 
    
    From $T^H(s',a)=t'$, by definition of $T^H$, it follows that there is some $T^{h_\ell}$ such that $s'=\bar{s}^\ell_m$, $t'=\bar{s}^\ell_{m+1}$ and $T^{h_\ell}(s',a)=t'$. As $\bar{s}^i_j=s'=\bar{s}^\ell_m$, we get $T^H(\bar{s}^i_j,a)=\bar{s}^\ell_{m+1}$. As actions are universally applicable, there is some $t\in S$ such that $T(s,a)=t$. Consider an execution trace of the form 
    \[\varepsilon = (s_0,a^i_0,\dots, a^i_{j-1},s, a, t, \dots)\]
   with $n$ actions. Note that $\varepsilon\in\textnormal{Tr}(\mathcal{D},s_0)$, so there is some $p\in\{1,\dots,k\}$ such that $\varepsilon=\varepsilon_p$. It thus follows that $t=s^p_{j+1}$ and $\tau_p \in \textnormal{ObsTr}(\mathcal{D},s_0)$.  From Claim 1, we get $\bar{s}^i_j=\bar{s}^p_j$, so from lines 6 and 11 of Algorithm \ref{algo:histories}, we get $T^{h_p}(\bar{s}^i_j,a)= \bar{s}^p_{j+1}$. By definition of $T^H$, we get $T^{H}(\bar{s}^i_j,a)= \bar{s}^p_{j+1}$. But since we already know that $T^H(\bar{s}^i_j,a)=\bar{s}^\ell_{m+1}$ and that $T_H$ is deterministic, we get $\bar{s}^p_{j+1}=\bar{s}^\ell_{m+1}$. Since $s=s^i_j$, $t=s^p_{j+1}$  and$t'=\bar{s}^\ell_{m+1}$, we get: $T(s,a)=t$ and $tZ t'$.

\end{description}

($\supseteq$) We show now that $\textsc{Domains}(P,\Sigma)\supseteq \{\mathcal{D}' \mid \mathcal{D}\bisim \mathcal{D}'\}$. By Lemma \ref{lemma:bisim_trace}, this is equivalent to showing $\textsc{Domains}(P,\Sigma)\supseteq \{((\states',\actions',\transmod',s'_0),\obs', \obsmod')\mid \textnormal{ObsTr}(\mathcal{D},s_0) = \textnormal{ObsTr}(\mathcal{D'},t_0)\}$, which is what we will do. Assume that  $((\states',\actions',\transmod',t_0),\obs, \obsmod)$ satisfies $\textnormal{ObsTr}(\mathcal{D},s_0) = \textnormal{ObsTr}(\mathcal{D'},t_0)$. Let $\tau_i\in \Sigma$. Since $\tau_i\in\textnormal{ObsTr}(\mathcal{D},s_0)$,$\tau_i\in\textnormal{ObsTr}(\mathcal{D}',s'_0)$. Since $s^i_0\in comp(s^i_0)$ and $s^i_1\in comp(s^i_1)$, looking at line 6 of Algorithm \ref{algo:histories}, we can see that at step $0$ of $\textsc{Histories}(P, \tau_i)$,  history $h=(s^i_0,Obs(s_0),a^i_0,s^i_{1},Obs(s_1))$ is created. At each step $j>0$, since $s^i_j\in comp(s^i_j)$ and $s^i_{j+1}\in comp(s^i_{j+1})$, looking at lines 12-13 of Algorithm \ref{algo:histories}, we can see that the extension \[h'=(s^i_0,Obs(s_0),a^i_0,s^i_{1},Obs(s_1)\dots,s^j,Obs(s^i_j),a^i_j,s^i_{j+1},Obs(s^i_{j+1}))\] of $h$ is such that $h'\in H_j$ (line 15 of Algorithm \ref{algo:histories}). For $\tau_i$, let $h_i$ denote the history \[h_i=(s^i_0,Obs(s_0),a^i_0,s^i_{1},Obs(s_1)\dots,s^i_{|S|},Obs(s^i_{|S|}),a^i_{|S|},s^i_n,Obs(s^i_n)).\] As $h_i$ is one of the histories created at the last step of the iteration in $\textsc{Histories}(P, \tau_i)$, $h_i$ is a member of the output of Algorithm \ref{algo:histories} (line 21). This process takes place for each $\tau_i\in \Sigma$ when $\textsc{Histories}(P, \tau_i)$ is executed, so the output of each call of $\textsc{Histories}(P, \tau_k)$ includes the history $h_i$, for $1\leq i\leq k$. Note that all $h_i$ start in the same state $s_0$. Moreover, since $\Sigma$ includes all observation traces with $2^{2|P|}$ actions, if $\mathcal{D}'$ has $2^{|P|}$ states, each state in $\mathcal{D}'$ is reached by one such observation trace. i.e.~for each transition $(s,a,t)\in T'$, there is some observation trace $\tau_i \in \Sigma$ is of the form $(\dots, Obs'(s),a,Obs'(t), \dots)$. Hence, $D^{h_i}$ has $(s,a,t)\in T^{h_i}$, $Obs^{h_i}(s)=Obs'(s)$ and $Obs^{h_i}(t)=Obs'(t)$. Thus, by construction, $\mathcal{D}^{\{h_1,\dots,h_n\}}$ will then contain all and only the states, transitions and state observations of $\mathcal{D}'$, which means that $\mathcal{D}^{\{h_1,\dots,h_n\}}= \mathcal{D}'$. And as $\mathcal{D}'$ is deterministic (since trace equivalence is defined for deterministic domains), this means that $\mathcal{D}^{\{h_1,\dots,h_n\}}$ is in the output of $\textsc{Domains}(P,\Sigma)$.
\end{proof}
Combining Lemma \ref{theorem:identif_trace} and Lemma \ref{lemma:bisim_trace}, we then immediately get that $\textsc{Domains}(P,\Sigma)$ returns  the set of domains over $P$ that are trace equivalent to $\mathcal{D}$.


\begin{theorem}\label{cor:L_1}
Consider the learner that on input $P$ (set of propositions) and $\Sigma$ (set of observation traces) returns the synchronous composition of the domains computed by running $\textsc{Domains}(P,\Sigma)$  (Algorithm \ref{algo:L1}). This learner outputs a domain that is behaviourally correct with respect to implicit knowledge about
$\mathcal{D}$.
\end{theorem} 
\begin{proof} From Lemma \ref{theorem:identif_trace}, we know that $\textsc{Domains}(P,\Sigma)$ returns the set of deterministic domains over $P$ that are bisimilar to $\mathcal{D}$, $\{\mathcal{D}' \mid \mathcal{D}\bisim \mathcal{D}'\}$. The learner then takes this set of domains and returns their synchronous composition. Since, by definition, the behavioural equivalence domain induced by $\mathcal{D}$ is the synchronous composition of 
$\{\mathcal{D}' \mid \mathcal{D}\bisim \mathcal{D}'\}$, the output of the learner is exactly the behavioural equivalence domain induced by $\mathcal{D}$.
\end{proof}

\section{Related work \label{sec:related}}

This work builds upon the framework by Bolander and Gierasimczuk~\cite{bolander2015learning,bolander2017learning}, where two basic learnability criteria for actions were studied: finite identifiability (conclusively inferring a representation of the correct action in finite time) and identifiability in the limit (inconclusive convergence to a representation of the right action). It has been shown that deterministic actions are finitely identifiable, while arbitrary (non-deterministic) actions are only identifiable in the limit, in the fully observable setting. Moreover, the paper presents exact learning algorithms for deterministic actions that produce DEL event models as output. Our work continues this line of research, since it aims at exact learning, it bounds the size of the required input, and it uses DEL as the action representation language (at least for the first of our learners). The main difference is that the present work considers learning in partially, rather than fully observable, domains. 

Our techniques are inspired by the tools of the theory of inductive inference (see, e.g., \cite{JORS99,OJMW97}). The condition of behaviourally correct learning is closely related to the classical behaviourally correct learning of recursive functions (\cite{Barzdin:1974wn}, see also \cite{Case:1983vs}). The capacity to extract non-explicit knowledge, created by requiring the learned structures to be deterministic, bears close resemblance to the increased power of learning when transitioning from the more general recursive language (set) learning \cite{gold1967language} to recursive function learning (for a comparison of the two in the context of BC-learning see \cite{Jain:2004tj}). The present paper is by no means the first to transfer techniques from inductive inference to the domain of DEL. A link was introduced in \cite{Gie09a,Gierasimczuk09}, where it was shown that finite identification \cite{mukouchi1992characterization,lange1992types} can be modelled in Public Announcement Logic \cite{Pla89}, and that the elimination process of learning by erasing \cite{Lange:1996aa} can be seen as iterated upgrade of dynamic doxastic logic \cite{Ben04}. The revival of finite identification resulted in designing new types of learners, such as preset learners and fastest learners, and gave new insights into the complexity of obtaining definite finite tell tales (DFTTs, also used our present paper) \cite{Gierasimczuk:2010aa,gierasimczuk2012complexity}. Some of those results were later used to investigate properties of finite identification from complete data \cite{Jongh:2019aa,Sandoval:2020ut}.
The more general approach of identifiability in the limit and its connections to doxastic upgrades allowed comparing the learning power of various belief revision methods \cite{BGS2019}, and gave topological characterisations of learnability \cite{BGS2015} followed by an introduction of a dynamic logic for learning theory (DLLT, see \cite{Baltag:2019aa}).  

There is a rich literature on learning symbolic action models from experience. Research on action learning began in the late 1980s and early 1990s, with systems such as LIVE \cite{shen1989rule}, EXPO \cite{gil1992acquiring} and OBSERVER \cite{wang1996learning}, which learned actions represented as STRIPS-like rules in fully observable domains. Since then, the literature has grown steadily, including some relatively recent overview papers~\cite{jimenez2012,arora2018}. Recent work can be compared along several dimensions, such as: the type of observations available to the learner (full, partial, or noisy); the type of actions learned (conditional, unconditional, stochastic, etc.); the methods used (inductive logic programming, neural networks, satisfiability techniques, etc.); and the learning guarantees provided by the learning algorithm (approximate or exact learning). 

We discuss those works that are most directly related to the results of this paper.

In recent years, several works have appeared that can learn action descriptions in partially observable environments \cite{shahaf2006learning,amir2008learning,yang2005learning,yang2007learning,zhuo2010learning,mourao2012learning,zhuo2009learning,molineaux2014learning,cresswell2013acquiring,mccluskey2009action,AINETO2019}. In these works, partial observability is induced by selecting at random $n<|P|$ propositional symbols to observe, for each state in the learning input. Each observation of a state $s$ in the learning input is subjected to this process independently, so what is observed about $s$ each time it is visited can be different. In other words, partial observability is unsystematic and observations can be thought of as random subsamples of the full state observation. No attempt is therefore made to \emph{learn} the observation function, as it is just modelled as a random corruption process. The treatment of partial observability in these papers is thus very different from ours. We assume that there is a domain-specific observation function, which is in fact deterministic, and consider the problem of learning \emph{both} the transition and observation functions. We don't only want to learn as much as possible about the underlying transition system; we also want to learn about our own observational limits (and hence the observational limits of other agents in the same state).

Besides treating partial observability differently, most of these works aim at approximate learning: their algorithms are experimentally evaluated, based on an error function, and typically offer no upper bounds on estimation errors. In contrast, we have learning goals that are exact in nature, and prove that the goals are guaranteed to be reached given certain inputs. Amongst learners for partially observable domains, work of
Amir, Chang and Shahaf~\cite{amir2008learning,shahaf2006learning} does present exact algorithms for identifying the effects and preconditions of deterministic actions in partially observable domains. The algorithms take an observation trace as input and return a set of deterministic action models that could have led to those observations. The output of the algorithms is exact in the sense of producing all and only those action models that could have led to these observations. However, as observations are random, no attempt is made to learn the observation function, and no bound is given in terms of the number of observations needed to reach their learning goal. Our work differs from theirs in three respects: learning about the observation function, representing actions using DEL, and characterising and bounding the number of observations needed to reach our learning goals.\footnote{We require sound and complete sets of observations for explicit learning, and sound and complete sets of observation traces for implicit learning.}

\section{Final remarks and future work\label{sec:conclusion}}

We conclude by reflecting on the assumptions made in this paper and exploring some avenues for future research.

\paragraph{Multi-agent learning} We have assumed in this paper that the environment involves a single agent. As we mentioned in Section \ref{sec:DEL}, the ultimate goal of the line of research introduced here is to be able to generalise to the multi-agent case, where a learner might end up learning not only what an action does and what is observed, but also what other agents will observe about such an action, and the knowledge or beliefs they will adopt as a result. Learning even just what is \emph{explicitly} known by others is hard, since it requires knowing what they are directly observing and how they are observing it. While an agent knows what it observes, it often doesn't know exactly what others observe. Sometimes common knowledge of what each agent observes can be achieved in a collaborative setting by a process such as \emph{joint attention} \cite{moore2014joint}, or by communicating what you observe. But in a non-cooperative setting or one in which communication is imperfect, the problem becomes highly non-trivial. Learning what others can observe, and what they believe or know as a result, is crucial for Theory of Mind reasoning~\cite{premack1978does} and epistemic planning~\cite{bolander2011epistemic} in unknown domains, and thus remains a key goal for future work.

\paragraph{Situated proactive learning}
We have assumed that the learner for implicit knowledge has access to all possible observation traces from the initial state $s_0$. These traces may have been generated by an expert agent or teacher that knows how to traverse the state space to produce such traces. For a situated, proactive learner, i.e.~a learner that has to gather such traces starting at $s_0$, the task of generating the traces is non-trivial and in some cases impossible. If the graph of the transition function is strongly connected (so that it is possible to get from every state to every other state through a sequence of actions), then all such observation traces can be generated from the initial state in one run through the graph, if the agent has a way of recognising the initial state each time it is revisited (it might not, due to the observational limitations). If the agent does not have a way of recognising the initial state, it is less clear how it would be able to explore the unknown graph of the transition function and reach a point in which it's certain that all possible traces have been produced. In some cases, doing so is impossible, e.g. the agent might get stuck in a `loopy' state from which all outgoing edges are loops. A solution to this might be to allow restarts as in reinforcement learning.

 The question of how to collect all observation traces as a situated proactive learner is closely related to the problem of \emph{exploring an unknown graph}~\cite{Deng:1999aa,albers2000exploring,PANAITE1999}.

\paragraph{Relaxing domain assumptions}
We have focused on deterministic domains, in which actions are always applicable and every state is reachable from the initial state $s_0$. In some cases, these assumptions may not be the most natural. Dropping some of them and generalising the learners in the paper is therefore a possible direction for future research. In some cases, e.g. in the learning algorithm for implicit knowledge, it may be possible to deal with non-deterministic domains by dropping the requirement that only deterministic domains are produced via Algorithms \ref{algo:histories} and \ref{algo:L1}. The algorithms would then produce a larger set of domains matching the observation traces, which includes non-deterministic ones.

\paragraph{Computational complexity}
 We have bounded the \emph{sample complexity} of our learning algorithms. The sample complexity of a learning algorithm is the size of the input required in order to achieve its learning goal. In the case of explicit knowledge, we have bounded the sample complexity of $\textsc{Learner}(P,A,\sigma)$ by the number of transitions in the system (i.e.~the size of a sound and complete set of observations). In the case of implicit knowledge, we have bounded the sample complexity of the learner of Theorem~\ref{cor:L_1} by the number of observation traces with $2^{2|P|}$ actions starting from $s_0$, i.e.~the size of a sound and complete set of observation traces).  On the other hand, we have left time and space complexity issues as future work. These complexities will of course depend on the implementation details. For example, the pseudo-code in Algorithm \ref{algo:histories} should not be implemented by creating the histories explicitly. There is a lot of repeated structure in histories, which can be avoided. Note that simply implementing the set of histories as paths in a tree, as depicted in Figure \ref{fig:histories}, would already yield an improvement, over storing each history separately. 
 This tree implementation is straightforward; more advanced and space-efficient ones may be possible, e.g. using action models, or some other compact encoding of domains.  

 \paragraph{A more compact representation of the learning output}
 The algorithm presented for implicit knowledge produces the behavioural equivalence domain as output. This domain can of course be very large, as its state space has size $O(2^{|P|})$. We leave for future work the task of learning the behavioural equivalence domain in a representation that is possibly more compact, such as a collection of DEL event models. This would then match the output we provided in the case of explicit knowledge. Ideally, we would compute one set of DEL event models that capture both implicit and explicit knowledge via two distinct indistinguishability relations. This would allow learners to reason about both their implicit and explicitly knowledge and how they are interrelated. An even more compact representation could possibly be obtained with first-order DEL (FODEL) action schemas such as those of Liberman et al.~\cite{liberman2020}, or with the succinct event models of Charrier and Schwarzentruber~\cite{charrier2017}. Since FODEL action schemas are represented with variables from first-order logic, learning them would require a non-trivial extension of existing methods. In a domain represented with first-order logic, the agent could possibly be observing ground atoms. It would then have to \emph{generalise} from these ground observation to achieve the level of abstraction characteristic of action schemas. Perhaps a procedure inspired in \emph{least general generalisation} \cite{plotkin1970note}, widely used for generalising first-order clauses in Inductive Logic Programming, could be adapted for this. Several algorithms for learning less expressive action models, such as those of the Planning Domain Definition Language (PDDL), also perform some type of generalisation, and could provide inspiration. 

\section{Acknowledgements}
Nina Gierasimczuk's research was funded by the Polish National Science Centre Grant 2015/19/B/HS1/03292.

\printbibliography

\end{document}